\documentclass{article}
\usepackage{iclr2027_conference,times}

\usepackage{microtype}
\usepackage{graphicx}
\usepackage{wrapfig}
\usepackage{placeins}
\usepackage{subcaption}
\usepackage{booktabs}
\usepackage{multirow}
\usepackage{xcolor}
\usepackage{hyperref}
\usepackage{bookmark}
\usepackage{url}
\usepackage{amsmath}
\usepackage{amssymb}
\usepackage{mathtools}
\usepackage{amsthm}
\usepackage{longtable}
\usepackage{pifont}
\usepackage[capitalize,noabbrev]{cleveref}
\usepackage[most]{tcolorbox}

\definecolor{takeawaygreen}{RGB}{81,116,52}
\definecolor{questionblue}{RGB}{69,148,193}
\definecolor{questionblue2}{RGB}{240,246,250}
\definecolor{imporvement}{RGB}{226,138,138}
\definecolor{decrease}{RGB}{107,158,219}
\newtcolorbox{takeaway}{
  enhanced,
  breakable,
  colback=white,
  colframe=takeawaygreen,
  boxrule=0pt,
  toprule=1pt,
  bottomrule=1pt,
  left=0pt,
  right=0pt,
  top=3pt,
  bottom=3pt,
}
\newtcolorbox{questionbox}{
  enhanced,
  breakable,
  colback=questionblue2,
  colframe=questionblue,
  boxrule=0pt,
  leftrule=4pt,
  toprule=0pt,
  bottomrule=0pt,
  rightrule=0pt,
  left=10pt,
  right=8pt,
  top=3pt,
  bottom=3pt,
}
\theoremstyle{plain}
\newtheorem{theorem}{Theorem}[section]
\newtheorem{proposition}[theorem]{Proposition}

\theoremstyle{definition}

\theoremstyle{remark}

\tcbuselibrary{skins}

\title{Understanding Gene Expression-to-Image Learning: Spatial Structure and Information Fidelity}

\author{
Xiaoqing Lian$^{1}$, Pei Liu$^{1}$, Jiahui Tan$^{1}$, Quan Zou$^{2}$, Chen Lin$^{3}$, Xiangxiang Zeng$^{4,5}$ \\
$^{1}$ College of Computer Science and Electronic Engineering, \\
\hspace*{1.5em} Hunan University, Changsha, China \\
$^{2}$ University of Electronic Science and Technology of China, \\
\hspace*{1.5em} Chengdu, China \\
$^{3}$ Xiamen University, Xiamen, China \\
$^{4}$ State Key Laboratory of Chemo and Biosensing, \\
\hspace*{1.5em} Institute of AI4Science, Hunan University, \\
\hspace*{1.5em} Changsha 410082, P. R. China \\
$^{5}$ Yuelushan Laboratory, China
}

\iclrfinalcopy

\begin{document}
\maketitle
\begin{abstract}
Single-cell RNA sequencing (scRNA-seq) enables cell-level analysis of biological heterogeneity, but learning effective representations from high-dimensional, unordered gene expression profiles remains challenging. Gene Expression-to-Image (GE2I) offers an effective learning strategy by arranging gene values on pixel grids for convolutional neural networks, but its underlying mechanisms remain unclear. We study GE2I on full-transcriptome data while keeping the network and training procedure fixed. Our results reveal a trade-off between spatial smoothness and gene-level fidelity. Structured mappings produce smooth, low-frequency patterns that support faster early optimization, whereas assigning multiple genes to the same pixel obscures expression differences and can reduce predictive performance. This trade-off explains why random mappings can compete with structured layouts when many genes are represented. One-to-one gene-to-pixel assignment improves prediction across the tested visual backbones even when image smoothness decreases. We further find that the benefits of pretraining and augmentation depend on the backbone and mapping. These findings provide practical guidance for pixel allocation and the transfer of vision priors to gene expression data.

\end{abstract}
\section{Introduction}

\begin{figure}[t]
    \centering
    \includegraphics[width=0.72\linewidth]{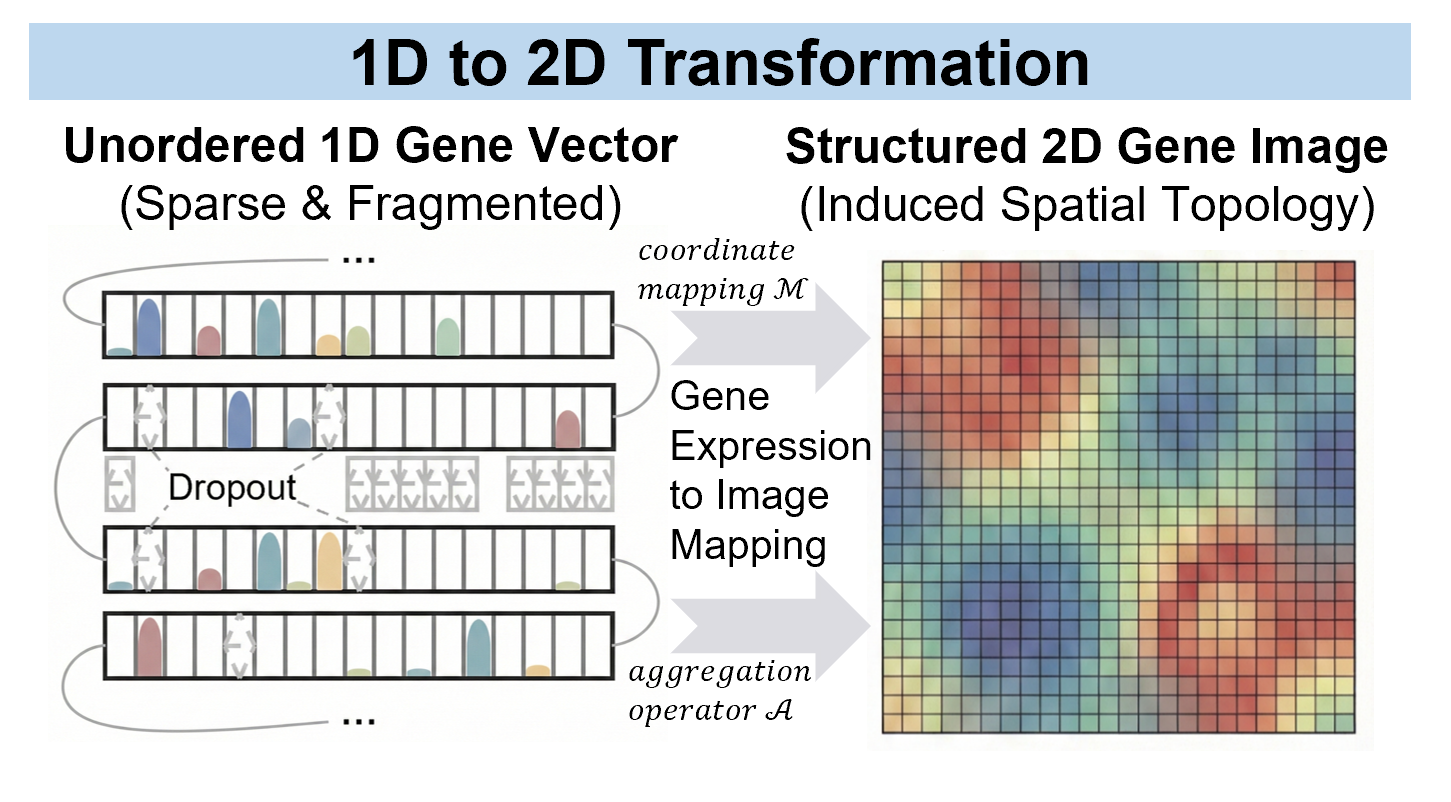}
    \caption{\textbf{Gene expression-to-image (GE2I).} A coordinate mapping $\mathcal{M}$ and aggregation operator $\mathcal{A}$ convert gene expression vectors into 2D images for a CNN. Colored gene labels illustrate local neighborhoods in the layout.}
    \label{Illustration}
\end{figure}
Single-cell RNA sequencing (scRNA-seq) measures gene expression in individual cells and helps researchers study how cells differ from one another~\citep{bunne2024build,eraslan2022single}. The resulting data form a high-dimensional matrix: each row represents a cell, and each column records the expression level of a gene~\citep{tang2009mrna}. Unlike pixels in an image or words in a sentence, the genes in each row have no natural spatial order. Computational models therefore need a useful way to represent these gene expression vectors~\citep{gunawan2023introduction}.

Common representation learning methods include probabilistic latent-variable models, such as Variational Autoencoders (VAEs)~\citep{kingma2013auto,gronbech2020scvae}, and self-attention-based sequence models, such as Transformers~\citep{vaswani2017attention,cui2024scgpt,zeng2025cellfm}. Another approach places the unordered gene values on a 2D pixel grid (Figure~\ref{Illustration}). We call this approach Gene Expression-to-Image (GE2I). Once the values are arranged as an image, a Convolutional Neural Network (CNN) can learn high-order statistics and nonlinear correlations without manually designed features~\citep{islam2023cartography}. Placing correlated genes together may also combine weak signals into clearer biological patterns. In addition, researchers can inspect the images directly to compare phenotypic differences across samples~\citep{sharma2019deepinsight}.

Most GE2I methods place functionally related genes near one another because they assume that biological neighborhoods improve the representation. Many mapping algorithms are designed around this idea. However, a structured layout changes two things at once: it preserves some gene relationships and gives the data a smoother spatial pattern. It is therefore unclear which property accounts for better learning. We ask:

\begin{questionbox} 
Which properties of gene expression images explain differences in optimization and predictive performance across mappings?
\end{questionbox}

We investigate this question on full-transcriptome data while holding the network and training conditions fixed. We measure image properties, track training behavior, and use beta regression to relate each property to prediction. The results show a trade-off between spatial smoothness and gene-level fidelity. A structured layout can produce patterns that a CNN learns quickly, but placing several genes in the same pixel merges their values and can remove useful differences. Faster early optimization may therefore lead to little or no improvement in final prediction. We test this explanation by changing how genes are assigned to pixels, then use the results to guide image construction and the use of pretrained vision models. Our contributions are:
\begin{questionbox}
\begin{itemize}
\item \textbf{Mechanistic analysis.} Image measurements, training dynamics, and beta regression show how spatial smoothness and gene-level fidelity affect early optimization and final prediction differently.

\item \textbf{Assignment interventions.} We test one-to-one gene-to-pixel assignments and introduce a greedy allocation procedure. Prediction improves even when these changes make the images less smooth.

\item \textbf{Backbone evaluation and transfer guidelines.} We compare assignment effects across visual backbones and test pretraining and augmentation. The results clarify when vision priors help with transcriptomic learning.

\end{itemize}
\end{questionbox}
\section{Related Work}

\subsection{Representation Learning in Single-Cell Genomics}

Representation learning in single-cell genomics aims to capture biological variation in sparse, high-dimensional expression profiles. Variational autoencoder models, including scVI~\citep{lopez2018deep} and scANVI~\citep{xu2021probabilistic}, learn probabilistic cell representations, with scANVI additionally incorporating available cell-type annotations. More recently, foundation models such as scGPT~\citep{cui2024scgpt}, scBERT~\citep{yang2022scbert}, and CellFM~\citep{zeng2025cellfm} have used large-scale pretraining to learn representations for downstream single-cell analysis. These approaches model expression profiles without first assigning genes to a two-dimensional pixel grid. GE2I introduces this spatial representation before applying a visual model, making the arrangement of gene measurements an additional part of representation design.

\subsection{The Gene Expression-to-Image (GE2I) Paradigm}

GE2I maps gene expression values onto two-dimensional grids for CNNs~\citep{sharma2019deepinsight,islam2023cartography}. Existing methods differ in how they construct these layouts. Projection-based approaches such as DeepInsight~\citep{sharma2019deepinsight} derive feature coordinates through dimensionality reduction, while MR-DeepInsight~\citep{sharma2024enhanced} combines multiple image representations. Fotomics~\citep{zandavi2023fotomics} instead uses Fourier transforms to determine feature coordinates. Other methods optimize spatial arrangements according to feature relationships: REFINED~\citep{bazgir2020representation} refines a low-dimensional embedding through local search, IGTD~\citep{zhu2021converting} matches the rankings of feature distances and pixel distances, and Genomap~\citep{islam2023cartography} uses optimal transport to encode gene relationships in a spatial representation.

Pixel allocation and visual training choices have also been considered in prior work. IGTD assigns each feature to a distinct pixel while optimizing neighborhood structure~\citep{zhu2021converting}. MR-DeepInsight uses pretrained CNNs and examines conventional image augmentation~\citep{sharma2024enhanced}. These developments address how to construct and learn from gene images, but comparing final prediction scores alone does not explain how the construction choices affect learning.

Our work complements these methods by examining spatial smoothness and gene-level fidelity. We relate these properties to early optimization and final prediction under shared network and training conditions, with pixel-assignment interventions providing a way to test the explanations.


\section{Analytical Framework}
\label{sec3}

\subsection{Problem Formulation}

Let $\mathbf{x}\in\mathbb{R}^{G}$ be a gene expression vector. A coordinate-based GE2I transformation uses a mapping $\mathcal{M}:\{1,\ldots,G\}\to\Omega$ onto an $H\times W$ grid and a permutation-invariant aggregation operator $\mathcal{A}$, such as the mean. At each occupied pixel, $\mathbf{I}(p)=\mathcal{A}(\{x_i\mid\mathcal{M}(i)=p\})$. The mapping is shared across cells. Here, ``topology'' refers to gene placement and neighborhood structure in the grid. Even when $HW\geq G$, a layout may assign multiple genes to the same pixel, combining their expression values. We therefore examine both spatial organization and gene-level fidelity.

\subsection{Spectral Descriptors (Quantifying Smoothness)}
\label{sec32}

We characterize spatial structure in the frequency domain~\citep{cao2022random}. Averaging the centered two-dimensional power spectrum within integer-radius rings gives a radial profile $\hat{r}\in\mathbb{R}^{L}$, where $L=\min(\lfloor W/2\rfloor,\lfloor H/2\rfloor)$. We normalize this profile so that $\sum_k\hat{r}_k=1$.

\textbf{Spectral Entropy (SE)} measures the spread of power across radial frequencies:
\begin{equation}
    \mathrm{SE}=-\sum_k\hat{r}_k\log(\hat{r}_k+\epsilon),
\end{equation}
where $\epsilon$ ensures numerical stability. Higher SE indicates a broader frequency distribution.

\textbf{Low-Frequency Energy Ratio (LFER)} measures low-frequency concentration:
\begin{equation}
    \mathrm{LFER}=\sum_{k=0}^{\lfloor L/4\rfloor}\hat{r}_k.
\end{equation}
Higher LFER indicates greater relative low-frequency power and smoother spatial patterns.

\subsection{Information Fidelity (Quantifying Resolution)}
\label{sec33}

We describe gene-level allocation using the occupied pixels $P_{used}\subseteq\Omega$ and the number of genes $g(p)$ assigned to each pixel.

\textbf{Average Pixels Per Gene (APPG)} measures occupied pixels relative to gene count:
\begin{equation}
    \mathrm{APPG}=\frac{|P_{used}|}{G}.
\end{equation}
For a single-view coordinate assignment, $\mathrm{APPG}=1$ means that each gene occupies a unique pixel.

\textbf{Collision Rate (CR)} measures the proportion of occupied pixels shared by multiple genes:
\begin{equation}
    \mathrm{CR}=\frac{|\{p\in P_{used}:g(p)>1\}|}{|P_{used}|}.
\end{equation}
A value of zero indicates that no genes share a pixel. For Genomap's soft transport coupling, uniform target marginals occupy all $G$ transport targets ($\mathrm{APPG}=1$), while entropic regularization gives each target contributions from multiple genes ($\mathrm{CR}=1$). Here, CR therefore denotes soft mixing rather than hard coordinate collisions. APPG and CR describe potential information loss through allocation, rather than directly measuring task-relevant information.

\subsection{The Convolutional Probe}

We use ResNet-18 as a common learner for cell-type annotation. Early loss reduction characterizes optimization, while accuracy and F1-Macro measure final prediction. We relate the four image descriptors jointly to these outcomes using beta regression~\citep{ferrari2004beta}. Each observation is a dataset--mapping--seed run within one input scope; Full and HVG inputs and the two prediction metrics are modeled separately:
\begin{align}
Y_{dms}&\sim\operatorname{Beta}\!\left(\mu_{dms}\phi,(1-\mu_{dms})\phi\right),\\
\operatorname{logit}(\mu_{dms})&=\alpha_{ds}+\beta_1\mathrm{APPG}_{dms}+\beta_2\mathrm{CR}_{dms}
+\beta_3\mathrm{LFER}_{dms}+\beta_4\mathrm{SE}_{dms},
\end{align}
where $d$, $m$, and $s$ index dataset, mapping, and seed. The descriptors enter on their original scales, $\alpha_{ds}$ denotes dataset-by-seed fixed effects, and $\phi$ is constant within each model. All outcomes lie strictly in $(0,1)$. Standard errors are clustered by dataset--method, and we report two-sided $p$-values together with Benjamini--Hochberg-adjusted $q$-values. Appendix~\ref{sec:appendix_stats} provides the inference, coverage, imputation, and diagnostic details. Pixel-assignment interventions then test whether changing allocation alters prediction.

\section{Experimental Protocols}

\subsection{Datasets and Benchmarks}

We evaluate cell-type classification on five datasets: Diabetic Kidney Disease (DKD)~\citep{wilson2022multimodal}, GTEx v9~\citep{eraslan2022single}, HypoMap~\citep{steuernagel2022hypomap}, Immune Cell Atlas~\citep{dominguez2022cross}, and Tabula Sapiens~\citep{the2022tabula}. These span disease-specific cohorts and cross-tissue atlases with varying class imbalance (Table~\ref{tab:datasets}). Our main analyses use full-transcriptome inputs, defined as all genes retained after preprocessing; complementary HVG evaluations are reported in the appendix. We also evaluate gene-to-protein prediction on NeurIPS 2021 CITE-Seq~\citep{stoeckius2017simultaneous}. Dataset specifications and evaluation protocols are provided in Appendix~\ref{appendix_data}.

\subsection{Gene Expression-to-Image Transformation Strategies}

The DeepInsight family (t-SNE, UMAP, PCA)~\citep{sharma2019deepinsight} and MR-DeepInsight~\citep{sharma2024enhanced} project genes onto fixed grids and average expression values at shared pixels. Fotomics~\citep{zandavi2023fotomics} uses FFT-based assignment with the same aggregation rule. Adaptive methods use grids with side length $\lceil\sqrt{G}\rceil$: Genomap~\citep{islam2023cartography} constructs optimal transport-weighted pixel sums, whereas IGTD~\citep{zhu2021converting} and REFINED~\citep{bazgir2020representation} optimize one-to-one assignments. Random Mapping assigns genes to unique coordinates without using gene similarity and keeps these coordinates fixed across cells.

Spectral descriptors are computed on native images; CNN inputs are resized to $224\times224$. Appendix~\ref{appendix_GE2I_protocols} details preprocessing, mapping, and normalization.

\subsection{Baseline Comparisons}

External methods provide performance context rather than form part of the controlled GE2I comparison. We include XGBoost on PCA features~\citep{chen2016xgboost}, Seurat~\citep{hao2021integrated}, scANVI with and without scArches~\citep{lotfollahi2020query}, and the foundation models scGPT~\citep{cui2024scgpt} and CellFM~\citep{zeng2025cellfm}. These models retain their respective input and training protocols; the foundation models are evaluated through zero-shot embeddings. A tuned MLP on raw expression vectors provides an additional non-image reference in the appendix. Implementation details appear in Appendix~\ref{appendixa2}.

\subsection{Training Protocol}

All main GE2I comparisons use the same ResNet-18 backbone~\citep{he2016deep} and training hyperparameters, without mapping-specific network tuning. We train with AdamW~\citep{loshchilov2017decoupled}, linear warmup, and cosine annealing, and report means and standard deviations across three random seeds. Hyperparameters are provided in Appendix~\ref{appendixa3}. Mapping construction is timed separately with a six-hour fitting budget. Additional experiments assess assignment effects across visual backbones and examine ImageNet pretraining and image augmentation.

\section{Results}
\label{sec:results}

We first evaluate prediction and mapping cost, then examine image properties and pixel assignment, and finally study visual backbones, pretraining, augmentation, and spatial interpretation.

\subsection{How Well Do Gene Expression Images Support Cell-Type Prediction?}
\label{sec:validity}
\label{sec:scalability}

\textbf{Gene expression images support competitive cell-type prediction.} Under the shared ResNet-18 setting, several GE2I representations perform comparably to established latent-variable methods, including MR-DeepInsight on the Immune Cell Atlas (Table~\ref{tab:hvg_results}). Single-view mappings also provide useful predictions, although scGPT remains stronger on several HVG datasets. External methods retain their recommended protocols, so this comparison evaluates the viability of GE2I rather than ranking fully optimized pipelines.

\begin{table}[!htbp]
\centering
\caption{\textbf{Cell-type prediction in the HVG setting.} GE2I methods use a shared ResNet-18 and training configuration. External methods follow their recommended implementations and settings. Entries report means and standard deviations over three seeds.}
\label{tab:hvg_results}
\resizebox{\linewidth}{!}{
\begin{tabular}{lcccccccccc}
\toprule
 & \multicolumn{2}{c}{DKD} & \multicolumn{2}{c}{GTEx v9} & \multicolumn{2}{c}{HypoMap} & \multicolumn{2}{c}{Immune Cell Atlas} & \multicolumn{2}{c}{Tabula Sapiens} \\
Method & ACC & F1-M & ACC & F1-M & ACC & F1-M & ACC & F1-M & ACC & F1-M \\
\midrule
XGBoost & $0.8646_{\pm0.01}$ & $0.6860_{\pm0.01}$ & $0.4509_{\pm0.04}$ & $0.1249_{\pm0.04}$ & $0.9661_{\pm0.00}$ & $0.3933_{\pm0.19}$ & $0.6040_{\pm0.03}$ & $0.4350_{\pm0.02}$ & $0.2387_{\pm0.08}$ & $0.0925_{\pm0.03}$ \\
scANVI & $0.9107_{\pm0.01}$ & $0.8082_{\pm0.03}$ & $0.5608_{\pm0.05}$ & $0.1694_{\pm0.06}$ & $0.9863_{\pm0.00}$ & $0.5061_{\pm0.23}$ & $0.7529_{\pm0.02}$ & $0.6207_{\pm0.02}$ & $0.3856_{\pm0.13}$ & $0.1484_{\pm0.06}$ \\
scANVI (scArches) & $0.9242_{\pm0.00}$ & $\mathbf{0.8424}_{\pm0.02}$ & $0.5691_{\pm0.05}$ & $0.1539_{\pm0.03}$ & $0.9866_{\pm0.00}$ & $0.5105_{\pm0.23}$ & $0.7606_{\pm0.02}$ & $0.6468_{\pm0.01}$ & $0.3236_{\pm0.02}$ & $0.1594_{\pm0.07}$ \\
Seurat & $0.8648_{\pm0.00}$ & $0.5908_{\pm0.03}$ & $0.5308_{\pm0.07}$ & $0.1305_{\pm0.05}$ & $0.9843_{\pm0.00}$ & $0.4782_{\pm0.23}$ & $0.5458_{\pm0.02}$ & $0.4208_{\pm0.02}$ & $0.3154_{\pm0.10}$ & $0.1287_{\pm0.05}$ \\
scGPT & $\mathbf{0.9421}_{\pm0.01}$ & $0.8286_{\pm0.02}$ & $\mathbf{0.7588}_{\pm0.01}$ & $\mathbf{0.2896}_{\pm0.07}$ & - & - & $\mathbf{0.7713}_{\pm0.02}$ & $\mathbf{0.7050}_{\pm0.01}$ & $\mathbf{0.4192}_{\pm0.16}$ & $0.2123_{\pm0.03}$ \\
CellFM & $0.6594_{\pm0.09}$ & $0.3744_{\pm0.03}$ & $0.4016_{\pm0.12}$ & $0.1143_{\pm0.05}$ & $0.6810_{\pm0.16}$ & $0.0852_{\pm0.00}$ & $0.4088_{\pm0.01}$ & $0.2101_{\pm0.01}$ & $0.0027_{\pm0.00}$ & $0.0011_{\pm0.00}$ \\
\midrule
Random (HVG) & $0.8996_{\pm0.00}$ & $0.7734_{\pm0.01}$ & $0.6530_{\pm0.05}$ & $0.2106_{\pm0.01}$ & $0.9795_{\pm0.01}$ & $0.5571_{\pm0.28}$ & $0.6828_{\pm0.03}$ & $0.4842_{\pm0.02}$ & $0.3837_{\pm0.08}$ & $0.2147_{\pm0.07}$ \\
DeepInsight (t-SNE) (HVG) & $0.8924_{\pm0.00}$ & $0.7424_{\pm0.01}$ & $0.6262_{\pm0.01}$ & $0.1943_{\pm0.01}$ & $0.9833_{\pm0.00}$ & $0.5720_{\pm0.22}$ & $0.6928_{\pm0.03}$ & $0.5040_{\pm0.01}$ & $0.3711_{\pm0.08}$ & $0.1844_{\pm0.04}$ \\
DeepInsight (UMAP) (HVG) & $0.9000_{\pm0.00}$ & $0.7587_{\pm0.02}$ & $0.6765_{\pm0.04}$ & $0.2209_{\pm0.01}$ & $0.9811_{\pm0.01}$ & $0.5783_{\pm0.25}$ & $0.7045_{\pm0.02}$ & $0.5100_{\pm0.00}$ & $0.3919_{\pm0.08}$ & $0.1967_{\pm0.06}$ \\
DeepInsight (PCA) (HVG) & $0.9029_{\pm0.00}$ & $0.7808_{\pm0.01}$ & $0.6230_{\pm0.02}$ & $0.1759_{\pm0.00}$ & $0.9833_{\pm0.00}$ & $0.5723_{\pm0.18}$ & $0.6858_{\pm0.03}$ & $0.4810_{\pm0.01}$ & $0.3742_{\pm0.06}$ & $0.2016_{\pm0.07}$ \\
GenoMAP (HVG) & $0.8808_{\pm0.03}$ & $0.7717_{\pm0.02}$ & $0.6439_{\pm0.04}$ & $0.2073_{\pm0.02}$ & $0.9766_{\pm0.01}$ & $0.5422_{\pm0.14}$ & $0.6683_{\pm0.02}$ & $0.4738_{\pm0.00}$ & $0.3611_{\pm0.09}$ & $0.2097_{\pm0.07}$ \\
IGTD (HVG) & $0.9022_{\pm0.00}$ & $0.7328_{\pm0.04}$ & $0.6660_{\pm0.04}$ & $0.2112_{\pm0.03}$ & $0.9836_{\pm0.00}$ & $0.5390_{\pm0.21}$ & $0.6900_{\pm0.02}$ & $0.4998_{\pm0.01}$ & $0.3928_{\pm0.08}$ & $0.2067_{\pm0.06}$ \\
Fotomics (HVG) & $0.9015_{\pm0.00}$ & $0.7424_{\pm0.02}$ & $0.6577_{\pm0.02}$ & $0.2007_{\pm0.00}$ & $0.9814_{\pm0.01}$ & $0.5200_{\pm0.23}$ & $0.7070_{\pm0.03}$ & $0.5104_{\pm0.00}$ & $0.3830_{\pm0.08}$ & $0.2086_{\pm0.07}$ \\
Refined (HVG) & $0.9023_{\pm0.01}$ & $0.7527_{\pm0.01}$ & $0.6369_{\pm0.07}$ & $0.1936_{\pm0.02}$ & $0.9810_{\pm0.00}$ & $0.5595_{\pm0.18}$ & $0.6627_{\pm0.00}$ & $0.4840_{\pm0.01}$ & $0.3725_{\pm0.05}$ & $\mathbf{0.2154}_{\pm0.07}$ \\
MR-DeepInsight (HVG) & $0.9085_{\pm0.01}$ & $0.7967_{\pm0.02}$ & $0.6842_{\pm0.04}$ & $0.1989_{\pm0.05}$ & $\mathbf{0.9922}_{\pm0.01}$ & $\mathbf{0.5932}_{\pm0.24}$ & $0.7482_{\pm0.02}$ & $0.6393_{\pm0.02}$ & $0.3872_{\pm0.08}$ & $0.2097_{\pm0.10}$ \\
\bottomrule
\end{tabular}
}
\end{table}

\textbf{Full-transcriptome inputs make smoothness and gene-collision differences more pronounced than HVG inputs.} For the main mechanistic analysis, we use all genes retained after preprocessing, avoiding dependence on a particular HVG selection rule. Each mapping must organize more expression values while the downstream CNN and input resolution remain fixed, making some construction characteristics more pronounced. Figure~\ref{fig:spectral_analysis} shows how spectral properties change with gene scope; complete measurements appear in Appendix Table~\ref{tab:full_metrics}.

\begin{figure}[!htbp]
  \centering
    \includegraphics[width=\linewidth]{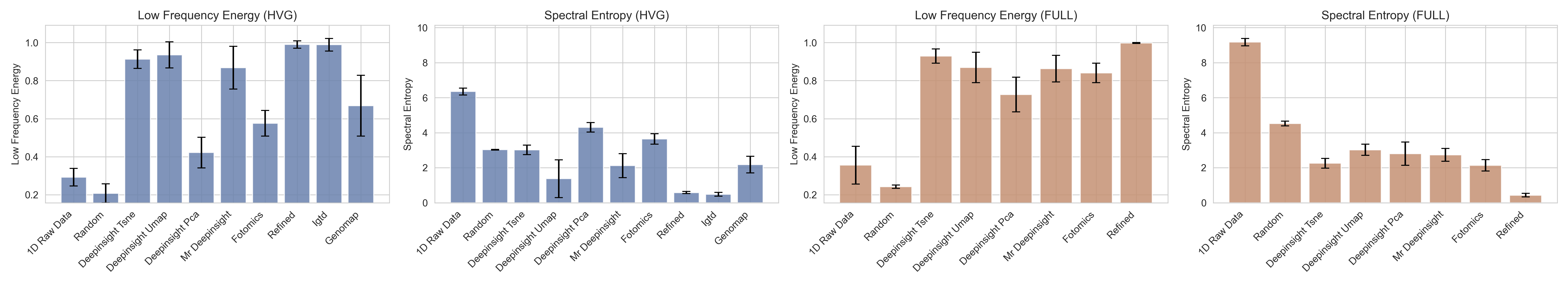}
    \caption{\textbf{Spectral properties of GE2I under HVG and full-transcriptome inputs.} LFER and SE characterize the native images produced by each mapping. Several structured mappings concentrate more power at low spatial frequencies than Random Mapping.}
    \label{fig:spectral_analysis}
\end{figure}

\begin{table}[!htbp]
\centering
\caption{\textbf{Full-transcriptome prediction.} Absolute accuracy and F1-Macro are reported under the shared GE2I learning configuration. scGPT is an external reference; the raw-expression MLP is reported in Appendix Table~\ref{tab:simple-mlp-full}. T/O denotes a fit exceeding the six-hour construction budget; dashes denote unreported results.}
\label{tab:full-scope-performance}
\resizebox{\linewidth}{!}{
\begin{tabular}{lcccccccccc}
\toprule
 & \multicolumn{2}{c}{DKD} & \multicolumn{2}{c}{GTEx v9} & \multicolumn{2}{c}{HypoMap} & \multicolumn{2}{c}{Immune Cell Atlas} & \multicolumn{2}{c}{Tabula Sapiens} \\
 Method & Acc & F1-M & Acc & F1-M & Acc & F1-M & Acc & F1-M & Acc & F1-M \\
\midrule
scGPT & $0.9462_{\pm0.00}$ & $0.8553_{\pm0.02}$ & $0.8091_{\pm0.02}$ & $0.3415_{\pm0.09}$ & - & - & $0.8061_{\pm0.03}$ & $0.7462_{\pm0.00}$ & $0.3930_{\pm0.15}$ & $0.1977_{\pm0.06}$ \\
\midrule
Random & $0.9421_{\pm0.00}$ & $0.8053_{\pm0.01}$ & $0.8489_{\pm0.03}$ & $0.4601_{\pm0.03}$ & $0.9901_{\pm0.00}$ & $0.6553_{\pm0.23}$ & $0.8785_{\pm0.00}$ & $0.7305_{\pm0.01}$ & $0.4569_{\pm0.10}$ & $\mathbf{0.3453}_{\pm0.11}$ \\
DeepInsight (t-SNE) & $0.9537_{\pm0.00}$ & $0.8868_{\pm0.02}$ & $0.8548_{\pm0.04}$ & $\mathbf{0.5064}_{\pm0.04}$ & $0.9918_{\pm0.00}$ & $0.7487_{\pm0.17}$ & $0.8766_{\pm0.01}$ & $0.7480_{\pm0.01}$ & $0.4270_{\pm0.13}$ & $0.2467_{\pm0.02}$ \\
DeepInsight (UMAP) & $0.9572_{\pm0.00}$ & $0.8924_{\pm0.00}$ & $0.8462_{\pm0.04}$ & $0.4885_{\pm0.03}$ & $0.9935_{\pm0.00}$ & $0.7713_{\pm0.15}$ & $0.8820_{\pm0.00}$ & $0.7625_{\pm0.02}$ & $0.4275_{\pm0.21}$ & $0.2672_{\pm0.06}$ \\
DeepInsight (PCA) & $0.9575_{\pm0.00}$ & $0.8458_{\pm0.00}$ & $0.8416_{\pm0.03}$ & $0.4415_{\pm0.01}$ & $0.9926_{\pm0.00}$ & $0.7276_{\pm0.19}$ & $0.8716_{\pm0.01}$ & $0.7483_{\pm0.03}$ & $0.4423_{\pm0.09}$ & $0.3054_{\pm0.11}$ \\
MR-DeepInsight & $\mathbf{0.9700}_{\pm0.00}$ & $\mathbf{0.9271}_{\pm0.01}$ & $\mathbf{0.8715}_{\pm0.04}$ & $0.3656_{\pm0.05}$ & $\mathbf{0.9994}_{\pm0.00}$ & $\mathbf{0.8021}_{\pm0.13}$ & $\mathbf{0.8990}_{\pm0.01}$ & $\mathbf{0.8224}_{\pm0.00}$ & $0.4097_{\pm0.16}$ & $0.2525_{\pm0.06}$ \\
Fotomics & $0.9543_{\pm0.00}$ & $0.8518_{\pm0.01}$ & $0.8327_{\pm0.03}$ & $0.4184_{\pm0.03}$ & $0.9913_{\pm0.00}$ & $0.7191_{\pm0.19}$ & $0.8613_{\pm0.01}$ & $0.7195_{\pm0.01}$ & $\mathbf{0.4601}_{\pm0.13}$ & $0.2947_{\pm0.05}$ \\
Refined & $0.9566_{\pm0.00}$ & $0.8704_{\pm0.02}$ & $0.8467_{\pm0.05}$ & $0.4875_{\pm0.04}$ & $0.9920_{\pm0.00}$ & $0.7274_{\pm0.18}$ & $0.8827_{\pm0.00}$ & $0.7302_{\pm0.01}$ & T/O & T/O \\
\bottomrule
\end{tabular}
}
\end{table}

Table~\ref{tab:full-scope-performance} reports full-transcriptome prediction, and Figure~\ref{method_fit_time_by_scope} reports construction time. DeepInsight and Fotomics fit within minutes to hours; IGTD timed out on all datasets, Genomap on all but DKD, and REFINED on Tabula Sapiens. Although the shared CNN processes images at the same resolution, mapping construction remains a separate computational constraint. We also evaluate GE2I on CITE-Seq gene-to-protein prediction (Appendix~\ref{sec:appendix_cross_modality}).

No layout performs best across datasets. DeepInsight (t-SNE) outperforms Random Mapping in F1-Macro on GTEx v9, whereas Random Mapping achieves the highest mean F1-Macro among the reported GE2I methods on Tabula Sapiens. A tuned expression-space MLP is competitive on DKD and the Immune Cell Atlas, while structured GE2I methods achieve higher F1-Macro on GTEx v9 and HypoMap (Appendix Table~\ref{tab:simple-mlp-full}). These differences motivate examining image properties and their relation to learning, rather than inferring the value of spatial organization from final scores alone.

\begin{figure}[!htbp]
\centering
\begin{minipage}[t]{0.48\linewidth}
  \vspace{0pt}
  \centering
  \includegraphics[width=\linewidth]{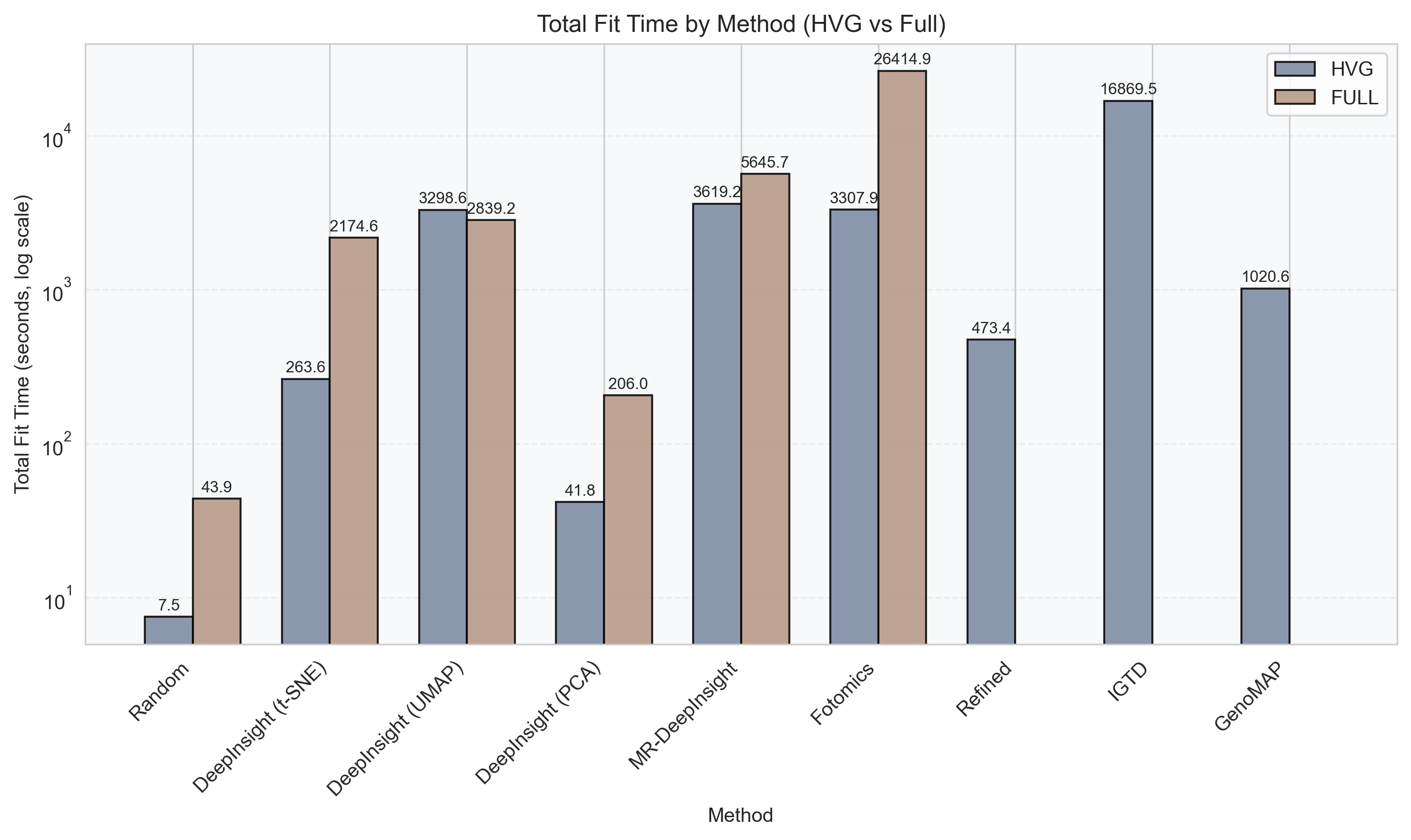}
  \caption{\textbf{Mapping construction cost.} Runtime to fit $\mathcal{M}$, aggregated across datasets for the plotted methods. Iterative strategies such as REFINED and Genomap are excluded from this summary because they exceeded the fitting budget on large benchmarks.}
  \label{method_fit_time_by_scope}
\end{minipage}\hfill
\begin{minipage}[t]{0.48\linewidth}
  \vspace{0pt}
  \centering
  \includegraphics[width=\linewidth]{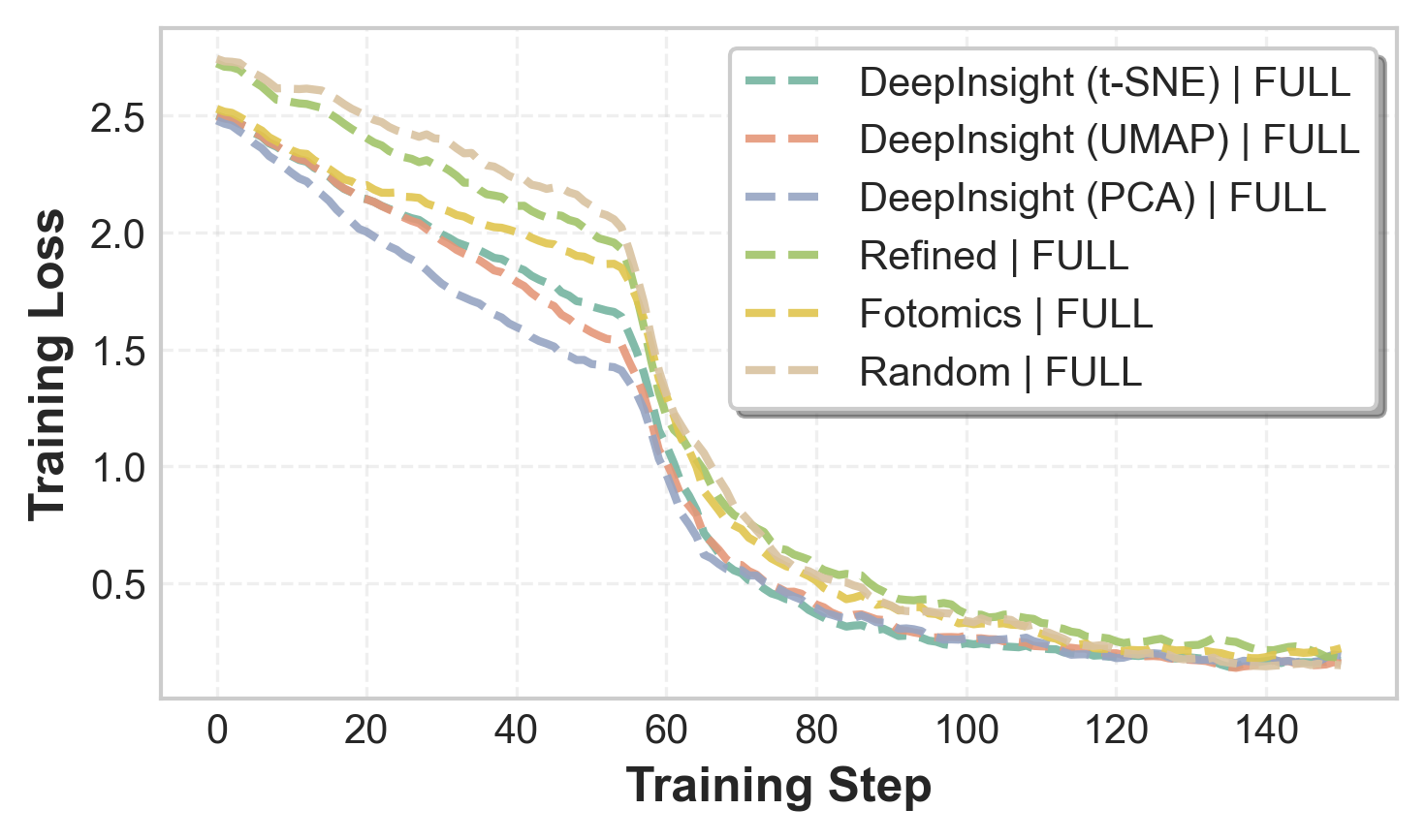}
  \captionof{figure}{\textbf{Early optimization on full-transcriptome DKD.} Training-loss curves over the first 60 steps under the shared CNN. Structured mappings show faster early loss reduction than Random Mapping.}
  \label{fig:loss_curves}
\end{minipage}
\end{figure}

\subsection{How Do Image Properties Relate to Optimization and Prediction?}
\label{sec:mechanism}

\textbf{Smoother images optimize faster, but preserving separate gene values is also important for final prediction.} We characterize these two effects using LFER and SE for spatial structure and APPG and CR for gene-level allocation. Appendix Figures~\ref{fig:hvg_cell_vision} and~\ref{fig:full_cell_vision} show the resulting image patterns, and Appendix Table~\ref{tab:full_metrics} provides the complete measurements.

\textbf{Structured mappings produce smoother, lower-frequency images than random placement.} DeepInsight (t-SNE/UMAP) and REFINED have full-transcriptome LFER values of approximately $0.8$--$1.0$, whereas Random Mapping remains near $0.24$ (Figure~\ref{fig:spectral_analysis}). Placing similarly expressed genes nearby produces coherent regions with gradually varying intensity. Random placement yields more irregular neighboring values and a broader frequency distribution. The CNN's local filters receive smoother patterns under structured layouts and more abrupt spatial changes under random placement.

\textbf{Low-frequency structure accelerates early optimization but does not ensure better final prediction.} On DKD, high-LFER mappings such as DeepInsight reduce loss faster over the first 60 training steps than Random Mapping (Figure~\ref{fig:loss_curves}). This is consistent with the preference for low-frequency patterns described in work on neural-network spectral bias~\citep{rahaman2019spectral}. Spatial organization can therefore aid optimization by presenting expression values in a form the CNN learns more readily. However, early loss reduction does not establish final predictive performance, motivating analysis of how individual gene values are represented.

\textbf{Joint regression links gene mixing to worse prediction and organized spatial structure to better prediction.} We fit beta regressions relating APPG, CR, LFER, and SE jointly to accuracy and F1-Macro in the Full setting (Table~\ref{tab:beta_regression}). Dataset-by-seed fixed effects account for baseline performance differences, while clustered standard errors account for repeated seeds within each dataset--method pair. Both outcomes show the same pattern: APPG and CR have negative coefficients, whereas LFER and SE have positive coefficients, with all eight associations remaining significant after multiple-testing adjustment. These results connect the two sides of the trade-off: gene mixing is associated with lower prediction quality, while organized spatial structure is associated with higher prediction quality, consistent with the representation-level mechanism derived in Appendix~\ref{app:theory}.

\begin{table}[!htbp]
\centering
\caption{\textbf{Beta regression on full-transcriptome performance.} Descriptors enter jointly with dataset-by-seed fixed effects, and SE is clustered by dataset--method. The $p$-values use a $t$ distribution with 34 degrees of freedom; $q$ denotes the Benjamini--Hochberg-adjusted $p$-value across the 16 descriptor tests. Coefficients are on the logit-mean scale.}
\label{tab:beta_regression}
\small
\begin{tabular}{llrrrrr}
\toprule
Target & Descriptor & Coef. & Cluster SE & $z^*$ & $p$ & $q$ \\
\midrule
Accuracy & APPG & -1.044 & 0.358 & -2.92 & 0.006 & 0.011 \\
         & CR   & -3.251 & 0.617 & -5.27 & $<0.001$ & $<0.001$ \\
         & LFER &  4.954 & 1.367 &  3.62 & $<0.001$ & 0.003 \\
         & SE   &  0.933 & 0.265 &  3.52 & 0.001 & 0.003 \\
\midrule
F1-Macro & APPG & -0.966 & 0.414 & -2.33 & 0.026 & 0.041 \\
         & CR   & -3.349 & 0.791 & -4.23 & $<0.001$ & 0.001 \\
         & LFER &  5.334 & 1.485 &  3.59 & 0.001 & 0.003 \\
         & SE   &  0.995 & 0.282 &  3.53 & 0.001 & 0.003 \\
\bottomrule
\end{tabular}
\end{table}

\textbf{More gene mixing is associated with worse prediction.} For F1-Macro, CR has a coefficient of $-3.349$ ($q=0.001$). When several genes share one pixel, their distinct expression values are replaced by a shared average, reducing gene-level separation.

\textbf{Organized low-frequency structure is associated with better prediction.} LFER has a coefficient of $+5.334$ ($q=0.003$), consistent with the smoother layouts and faster early optimization observed above. SE is also positive ($+0.995$, $q=0.003$). Together, these results favor images that are spatially smooth but still contain variation across scales, rather than images that are irregular or nearly uniform.

\textbf{Using more pixels alone does not guarantee better prediction.} APPG has a coefficient of $-0.966$ ($q=0.041$) after holding CR and spectral descriptors fixed. Performance depends on which genes share pixels and how the values are arranged. Accuracy shows the same directions. Coefficient magnitudes are on the logit-mean scale; Appendix~\ref{sec:appendix_stats} reports the HVG results and diagnostics.

\textbf{Random layouts can compete by preserving gene identity in separate pixels.} Random Mapping forgoes similarity-based neighborhoods but preserves gene identity through fixed, separate pixels across cells. This helps explain its competitive performance against smoother mappings that merge expression values. Broader gene coverage accentuates this distinction: for DeepInsight (t-SNE) on HypoMap, CR rises from $0.198$ with HVGs to $0.641$ with full-transcriptome inputs. IGTD and REFINED show that spatial organization and collision-free allocation can coexist. If aggregation limits a structured mapping's benefit, separating genes should improve prediction even when the image becomes less smooth.

\subsection{Can Changing Pixel Assignment Improve Prediction?}
\label{sec:allocation}

\textbf{Yes. Assigning each gene to a unique pixel improves prediction even when image smoothness decreases.} Across the tested DeepInsight mappings and datasets, removing collisions increases F1-Macro while LFER decreases. In contrast, simply enlarging the image leaves substantial gene overlap and provides little improvement.

\begin{table}[t]
\centering
\caption{\textbf{Image-size and pixel-assignment experiments on DKD.} (a) Prediction and allocation descriptors across the reported image sizes. (b) LSA results, with superscripts indicating relative changes from the original assignment.}
\label{tab:resolution_allocation}
\begin{subtable}[t]{0.48\linewidth}
\centering
\caption{Image-size ablation}
\resizebox{\linewidth}{!}{
\begin{tabular}{lcccc}
\toprule
Size & Acc & F1-M & APPG & CR \\
\midrule
$64^2$ & $0.9608$ & $0.8990$ & $0.480$ & $0.493$ \\
$128^2$ & $0.9570$ & $0.8913$ & $0.489$ & $0.481$ \\
$224^2$ & $0.9556$ & $0.8915$ & $0.486$ & $0.485$ \\
$384^2$ & $0.9597$ & $0.9013$ & $0.478$ & $0.495$ \\
\bottomrule
\end{tabular}}
\label{tab:canvas_size}
\end{subtable}\hfill%
\begin{subtable}[t]{0.48\linewidth}
\centering
\caption{Efficacy of linear sum assignment (LSA)}
\resizebox{\linewidth}{!}{
\begin{tabular}{lcccc}
\toprule
Method & LFER & CR & Acc & F1-M \\
\midrule
DeepInsight (t-SNE) & $0.90^{\textcolor{decrease}{-4\%}}$ & $\mathbf{0.00}^{\textcolor{decrease}{-100\%}}$ & $0.97^{\textcolor{imporvement}{+1\%}}$ & $0.91^{\textcolor{imporvement}{+3\%}}$ \\
DeepInsight (UMAP) & $0.84^{\textcolor{decrease}{-7\%}}$ & $\mathbf{0.00}^{\textcolor{decrease}{-100\%}}$ & $0.97^{\textcolor{imporvement}{+1\%}}$ & $0.91^{\textcolor{imporvement}{+2\%}}$ \\
\bottomrule
\end{tabular}}
\label{tab:lsa_efficacy}
\end{subtable}
\end{table}

\textbf{Larger canvases neither remove gene collisions nor meaningfully improve prediction.} In the DKD ablation (Table~\ref{tab:canvas_size}), increasing image dimensions from $64^2$ to $384^2$ changes F1-Macro from $0.8990$ to $0.9013$, with lower scores at intermediate sizes. CR remains near $0.49$ throughout. Larger images thus provide little improvement while substantial gene overlap persists, motivating a change to pixel allocation with the input genes and downstream learner fixed.

\textbf{One-to-one assignment improves prediction despite reducing image smoothness.} Linear sum assignment (LSA)~\citep{burkard1980linear} gives each gene a distinct pixel. On DKD, it eliminates native-image collisions and improves F1-Macro for DeepInsight (t-SNE) and DeepInsight (UMAP) by approximately $3\%$ and $2\%$, respectively, while LFER decreases by $4$--$7\%$ (Table~\ref{tab:lsa_efficacy}). Consistent with the prediction from Section~\ref{sec:mechanism}, separating gene values improves prediction despite lower image smoothness: the benefit of revised allocation can outweigh reduced low-frequency concentration.

\textbf{Greedy unique assignment removes collisions and improves all six mapping--dataset comparisons.} We introduce a greedy assignment procedure that avoids a dense gene--pixel cost matrix. Genes are processed in a fixed order, each placed at the nearest available pixel to its rounded target coordinate, with an occupancy mask preventing reuse. This preserves projected coordinates as placement targets while enforcing distinct locations (Appendix~\ref{app:unique_assignment}).

Table~\ref{tab:lsa-vs-bin-resnet18} reports reassignment results for DeepInsight (t-SNE) and DeepInsight (UMAP) on DKD, GTEx v9, and HypoMap. All six comparisons improve mean F1-Macro, with $\mathrm{CR}=0$ and $\mathrm{APPG}=1$ on reassigned native images. For example, DeepInsight (t-SNE) on GTEx v9 improves from approximately $0.506$ to $0.556$, while LFER falls from $0.867$ to $0.756$. LFER decreases in every comparison, so low-frequency concentration does not explain the gains.

These interventions support reducing unnecessary gene aggregation while using the projected layout to guide placement. Reassignment changes both gene separation and local spatial patterns, and this combination improves prediction despite lower smoothness across the tested mappings. We next examine whether the benefit extends beyond ResNet-18 and how other visual training choices interact with GE2I.

\begin{table}[!htbp]
\centering
\caption{\textbf{Greedy unique assignment across datasets.} Results use ResNet-18. Superscripts report relative changes from conventional binning for the corresponding dataset and mapping. The algorithm is described in Appendix~\ref{app:unique_assignment}.}
\label{tab:lsa-vs-bin-resnet18}
\resizebox{\columnwidth}{!}{
\begin{tabular}{llcccccc}
\toprule
Dataset & Mapping & APPG & CR & LFER & SE & Acc & F1-M \\
\midrule
DKD & DeepInsight (t-SNE) & $1.0000_{\pm 0.00}^{\textcolor{imporvement}{+104.8\%}}$ & $0.0000_{\pm 0.00}^{\textcolor{decrease}{-100.0\%}}$ & $0.8021_{\pm 0.00}^{\textcolor{decrease}{-15.0\%}}$ & $2.3435_{\pm 0.00}^{\textcolor{imporvement}{+19.7\%}}$ & $0.9679_{\pm 0.00}^{\textcolor{imporvement}{+1.5\%}}$ & $0.9047_{\pm 0.00}^{\textcolor{imporvement}{+2.0\%}}$ \\
DKD & DeepInsight (UMAP) & $1.0000_{\pm 0.00}^{\textcolor{imporvement}{+84.8\%}}$ & $0.0000_{\pm 0.00}^{\textcolor{decrease}{-100.0\%}}$ & $0.7811_{\pm 0.00}^{\textcolor{decrease}{-13.5\%}}$ & $2.4673_{\pm 0.00}^{\textcolor{decrease}{-10.0\%}}$ & $0.9667_{\pm 0.00}^{\textcolor{imporvement}{+1.0\%}}$ & $0.9006_{\pm 0.02}^{\textcolor{imporvement}{+0.9\%}}$ \\
GTEx v9 & DeepInsight (t-SNE) & $1.0000_{\pm 0.00}^{\textcolor{imporvement}{+142.0\%}}$ & $0.0000_{\pm 0.00}^{\textcolor{decrease}{-100.0\%}}$ & $0.7560_{\pm 0.04}^{\textcolor{decrease}{-12.8\%}}$ & $2.7173_{\pm 0.30}^{\textcolor{imporvement}{+2.8\%}}$ & $0.8630_{\pm 0.03}^{\textcolor{imporvement}{+1.0\%}}$ & $0.5557_{\pm 0.02}^{\textcolor{imporvement}{+9.7\%}}$ \\
GTEx v9 & DeepInsight (UMAP) & $1.0000_{\pm 0.00}^{\textcolor{imporvement}{+271.3\%}}$ & $0.0000_{\pm 0.00}^{\textcolor{decrease}{-100.0\%}}$ & $0.7374_{\pm 0.05}^{\textcolor{decrease}{-8.0\%}}$ & $2.8200_{\pm 0.32}^{\textcolor{decrease}{-12.0\%}}$ & $0.8495_{\pm 0.04}^{\textcolor{imporvement}{+0.4\%}}$ & $0.5357_{\pm 0.05}^{\textcolor{imporvement}{+9.7\%}}$ \\
HypoMap & DeepInsight (t-SNE) & $1.0000_{\pm 0.00}^{\textcolor{imporvement}{+229.2\%}}$ & $0.0000_{\pm 0.00}^{\textcolor{decrease}{-100.0\%}}$ & $0.8602_{\pm 0.00}^{\textcolor{decrease}{-9.4\%}}$ & $2.0356_{\pm 0.07}^{\textcolor{decrease}{-7.7\%}}$ & $0.9940_{\pm 0.00}^{\textcolor{imporvement}{+0.2\%}}$ & $0.7945_{\pm 0.13}^{\textcolor{imporvement}{+6.1\%}}$ \\
HypoMap & DeepInsight (UMAP) & $1.0000_{\pm 0.00}^{\textcolor{imporvement}{+359.5\%}}$ & $0.0000_{\pm 0.00}^{\textcolor{decrease}{-100.0\%}}$ & $0.8085_{\pm 0.05}^{\textcolor{decrease}{-11.7\%}}$ & $2.2516_{\pm 0.34}^{\textcolor{decrease}{-21.8\%}}$ & $0.9938_{\pm 0.00}^{\textcolor{imporvement}{+0.0\%}}$ & $0.8137_{\pm 0.13}^{\textcolor{imporvement}{+5.5\%}}$ \\
\bottomrule
\end{tabular}
}
\end{table}

\subsection{How Do Visual Backbones and Training Choices Affect GE2I?}
\label{sec:cv_priors}

\textbf{Unique pixel assignment helps across backbones, whereas pretraining and augmentation depend on the architecture and mapping.} On DKD, reassignment improves mean F1-Macro for ResNet-18, ResNet-50, ConvNeXt-Tiny, EfficientNet-B0, and ViT-B/16 by $3.4\%$--$7.0\%$ (Appendix Table~\ref{tab:dkd-backbone-lsa-shared}). The consistent gains across convolutional and transformer models show that the allocation benefit is not specific to ResNet-18. Appendix Table~\ref{tab:random-backbones-dkd} provides the corresponding Random Mapping results.

\textbf{ImageNet pretraining accelerates early ResNet-18 learning for structured images, but not for random layouts.} ImageNet initialization~\citep{deng2009imagenet} has different effects across mappings and architectures. In the ResNet-18 training curves on DKD, pretrained weights accelerate early learning for structured DeepInsight images, whereas Random Mapping shows no clear acceleration (Figure~\ref{pretrained_loss_curves}). This contrast is consistent with pretrained visual features being better suited to the smoother spatial patterns of structured layouts.

Final prediction depends on the backbone (Table~\ref{tab:backbone_pretrained}). ImageNet pretraining improves mean F1-Macro for ViT-B/16 and ResNet-50 under both DeepInsight (t-SNE) and REFINED, whereas ResNet-18 performs better with random initialization. Thus, some learners benefit from pretraining even when the smaller CNN learns effectively from expression images alone. Together with the training curves, these results indicate architecture- and mapping-dependent benefits; early acceleration alone does not establish a final predictive advantage.

\begin{figure}[!htbp]
\centering
\begin{minipage}[t]{0.47\textwidth}
  \vspace{0pt}
  \centering
  \includegraphics[width=\linewidth]{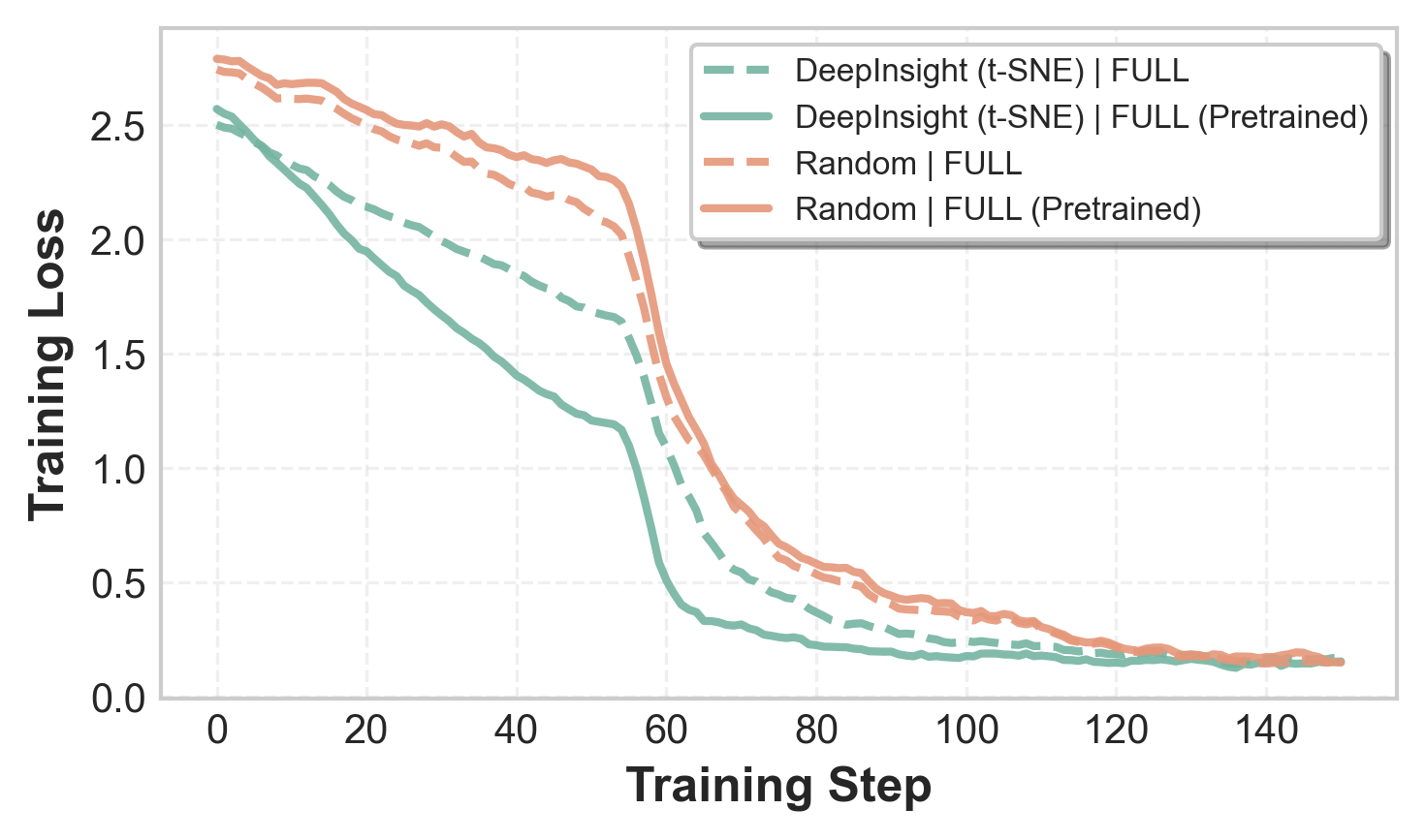}
  \captionof{figure}{\textbf{ImageNet initialization and early learning on DKD.} ResNet-18 training curves with and without pretrained weights. Structured DeepInsight images benefit in early optimization, whereas Random Mapping shows no clear acceleration.}
  \label{pretrained_loss_curves}
\end{minipage}\hfill
\begin{minipage}[t]{0.49\textwidth}
  \vspace{0pt}
  \centering
  \captionof{table}{\textbf{Backbones and pretraining on DKD.} Accuracy and F1-Macro with ImageNet initialization and random initialization for two structured mappings.}
  \label{tab:backbone_pretrained}
  \resizebox{\linewidth}{!}{
  \begin{tabular}{llc c c}
  \toprule
  Method & Backbone & Pretrained & ACC & F1-Macro\\
  \midrule
  \multirow{6}{*}{DeepInsight (t-SNE)} & \multirow{2}{*}{ResNet-18} & Yes & $0.9479_{\pm 0.01}$ & $0.8767_{\pm 0.01}$\\
   &  & No & $0.9537_{\pm 0.00}$ & $0.8868_{\pm 0.02}$\\
   & \multirow{2}{*}{ResNet-50} & Yes & $0.9512_{\pm 0.00}$ & $0.8915_{\pm 0.02}$\\
   &  & No & $0.9457_{\pm 0.01}$ & $0.8676_{\pm 0.01}$\\
   & \multirow{2}{*}{ViT-B/16} & Yes & $0.9510_{\pm 0.00}$ & $0.8659_{\pm 0.03}$\\
   &  & No & $0.9401_{\pm 0.01}$ & $0.8252_{\pm 0.01}$\\
  \midrule
  \multirow{6}{*}{Refined} & \multirow{2}{*}{ResNet-18} & Yes & $0.9504_{\pm 0.01}$ & $0.8392_{\pm 0.02}$\\
   &  & No & $0.9566_{\pm 0.00}$ & $0.8704_{\pm 0.03}$\\
   & \multirow{2}{*}{ResNet-50} & Yes & $0.9452_{\pm 0.01}$ & $0.8266_{\pm 0.02}$\\
   &  & No & $0.9468_{\pm 0.00}$ & $0.8159_{\pm 0.02}$\\
   & \multirow{2}{*}{ViT-B/16} & Yes & $0.9529_{\pm 0.00}$ & $0.8460_{\pm 0.01}$\\
   &  & No & $0.9505_{\pm 0.01}$ & $0.8294_{\pm 0.04}$\\
  \bottomrule
  \end{tabular}}
\end{minipage}
\end{figure}

\textbf{Augmentation: geometry versus intensity.} On DKD, geometric transformations lower mean F1-Macro for DeepInsight (t-SNE) and REFINED but improve Random Mapping (Appendix Table~\ref{tab:augmentation_ablation_dkd}). Because these transformations change the locations associated with gene identities, their suitability depends on the mapping and cannot be inferred solely from their usefulness for natural images.

Intensity perturbations vary brightness and contrast without changing gene coordinates and improve mean F1-Macro for all three mappings. Stronger perturbation is not uniformly better: combining all augmentations lowers F1-Macro for DeepInsight (t-SNE) and REFINED to $0.3227$ and $0.1941$, respectively, whereas Random Mapping improves under every tested setting. Intensity perturbation is thus the most consistent choice in this comparison; geometric and combined augmentations require mapping-specific evaluation.

\textbf{Spatial interpretation beyond prediction.} Competitive prediction from random layouts does not make spatial organization irrelevant to interpretation. On DKD, DeepInsight (t-SNE) has a Grad-CAM marker hit ratio of $0.50$, versus $0.25$ for Random Mapping (Appendix Table~\ref{tab:dkd_gradcam_interpretability}). Its salient regions contain five of five listed markers for distal convoluted tubule cells and podocytes, versus two of five for Random Mapping (Appendix Figure~\ref{fig:dkd_gradcam_examples}). These more marker-aligned regions support gene-level inspection in this analysis. 

\section{Conclusion and Limitations}

Under controlled training conditions, GE2I exhibits a trade-off between spatial smoothness and gene-level fidelity. Structured mappings produce smooth patterns that can ease early CNN optimization, but assigning multiple genes to the same pixel may obscure expression differences and reduce final predictive performance. Our pixel-assignment experiments show that preserving gene-level information improves prediction even when the images become less smooth. Transcriptomic image design should therefore balance spatial organization with information preservation rather than treat smoothness or early optimization as a proxy for prediction quality.

\textbf{Limitations and Future Directions.}
Our experiments focus mainly on cell-type prediction with a shared CNN and training protocol. This setting isolates representation-level effects but does not measure the best performance attainable through method-specific tuning. Our pixel-level descriptors also do not capture higher-order relationships among genes. Further work should learn adaptive allocations that balance spatial organization, gene-level fidelity, and interpretability, and evaluate them on larger multimodal datasets and with foundation models.

\section*{AI Use Statement}

Generative AI tools were used to support research ideation, literature discovery, manuscript drafting, and language polishing. AI-assisted suggestions were reviewed against the experimental results and relevant literature. All experiments, numerical results, and statistical analyses were conducted and verified by the authors. The authors take full responsibility for the final content of this work.

\section*{Ethics Statement}

This study uses previously released single-cell transcriptomic datasets and does not involve the collection of new human-subject data. The analyses are intended for computational genomics research rather than clinical diagnosis or treatment.

\section*{Reproducibility Statement}

Experimental settings, preprocessing procedures, mapping parameters, training configurations, statistical analyses, and pixel-assignment methods are described in the main text and appendix. Results are reported across three random seeds where applicable. Source code will be released upon publication.

\FloatBarrier
\bibliography{paper}
\bibliographystyle{iclr2027_conference}

\newpage
\appendix
\onecolumn

\section{A Theoretical View of the Smoothness--Fidelity Trade-off}
\label{app:theory}

The empirical analysis suggests that spatial smoothness can facilitate learning,
while gene mixing can reduce predictive information. We illustrate how these
two effects can arise from the same pixel-aggregation operation.

Consider a binary classification problem with $Y\in\{-1,+1\}$. For the
$k_p$ genes assigned to pixel $p$, let
\begin{equation}
    \mathbf{X}_p \mid Y=y
    =
    y\boldsymbol{\mu}_p+\boldsymbol{\varepsilon}_p,
    \qquad
    \boldsymbol{\varepsilon}_p
    \sim\mathcal{N}(0,\sigma^2 I_{k_p}),
    \label{eq:block_model}
\end{equation}
where noise is independent across genes and pixels. Mean aggregation produces
\begin{equation}
    Z_p
    =
    \frac{1}{k_p}\mathbf{1}^{\top}\mathbf{X}_p.
\end{equation}
Let $\bar{\mu}_p=k_p^{-1}\mathbf{1}^{\top}\boldsymbol{\mu}_p$.

\begin{proposition}[Smoothness--fidelity effect of pixel aggregation]
\label{prop:smoothness_fidelity}
Mean aggregation has two simultaneous effects.

First, for neighboring pixels $p$ and $q$,
\begin{equation}
    \mathbb{E}
    \left[
        (Z_p-Z_q)^2 \mid Y
    \right]
    =
    (\bar{\mu}_p-\bar{\mu}_q)^2
    +
    \sigma^2
    \left(
        \frac{1}{k_p}+\frac{1}{k_q}
    \right).
    \label{eq:local_smoothness}
\end{equation}
Thus, holding the underlying spatial signal fixed, averaging more genes
suppresses noise-induced local variation.

Second, let $D_X^2$ and $D_Z^2$ denote the squared Mahalanobis distances
between the two classes before and after aggregation. Then
\begin{align}
    D_X^2
    &=
    \frac{4}{\sigma^2}
    \sum_p
    \|\boldsymbol{\mu}_p\|_2^2,\\
    D_Z^2
    &=
    \frac{4}{\sigma^2}
    \sum_p
    k_p\bar{\mu}_p^2,
\end{align}
and therefore
\begin{equation}
    D_X^2-D_Z^2
    =
    \frac{4}{\sigma^2}
    \sum_p
    \left\|
        \boldsymbol{\mu}_p
        -
        \bar{\mu}_p\mathbf{1}
    \right\|_2^2
    \geq 0.
    \label{eq:discriminative_loss}
\end{equation}
Aggregation therefore removes exactly the class-discriminative components
that distinguish genes within the same pixel.
\end{proposition}

\begin{proof}
For each pixel,
\[
Z_p
=
y\bar{\mu}_p
+
\bar{\varepsilon}_p,
\qquad
\operatorname{Var}(\bar{\varepsilon}_p)
=
\frac{\sigma^2}{k_p}.
\]
Independence across pixels immediately gives
Eq.~\eqref{eq:local_smoothness}.

Before aggregation, the two class means differ by
$2\boldsymbol{\mu}_p$ within pixel $p$, and the covariance is
$\sigma^2 I$, giving
\[
D_X^2
=
\frac{4}{\sigma^2}
\sum_p
\|\boldsymbol{\mu}_p\|_2^2.
\]
After aggregation, the class means at pixel $p$ differ by
$2\bar{\mu}_p$, while its variance is $\sigma^2/k_p$. Hence
\[
D_Z^2
=
\frac{4}{\sigma^2}
\sum_p
k_p\bar{\mu}_p^2.
\]
Finally, decompose
\[
\boldsymbol{\mu}_p
=
\bar{\mu}_p\mathbf{1}
+
\boldsymbol{\mu}_p^\perp,
\qquad
\mathbf{1}^{\top}\boldsymbol{\mu}_p^\perp=0.
\]
Orthogonality gives
\[
\|\boldsymbol{\mu}_p\|_2^2
=
k_p\bar{\mu}_p^2
+
\|\boldsymbol{\mu}_p^\perp\|_2^2,
\]
which yields Eq.~\eqref{eq:discriminative_loss}.
\end{proof}

The proposition provides a simple explanation for the empirical
smoothness--fidelity trade-off. Averaging reduces gene-specific fluctuations,
thereby lowering local spatial variation. On a regular grid, the sum of
neighboring squared differences is the discrete Dirichlet energy, which
penalizes rapidly varying spatial components and therefore connects this
effect to the low-frequency structure measured in the main analysis.
At the same time, averaging removes within-pixel gene contrasts that can
carry class-discriminative information.

This also clarifies why the two properties need not be fundamentally
incompatible. If genes grouped within a pixel have the same class-dependent
signal, i.e.,
$\boldsymbol{\mu}_p=\bar{\mu}_p\mathbf{1}$, then
Eq.~\eqref{eq:discriminative_loss} is zero even though averaging still reduces
noise-induced spatial variation. More generally, spatial smoothness can also
be obtained by placing similarly behaving genes nearby without merging their
values.

The result corresponds directly to the two dominant effects observed in the
joint regression analysis: greater low-frequency organization is associated
with better prediction, whereas greater gene mixing is associated with worse
prediction. The theory does not prescribe regression coefficients; rather, it
shows how these opposite associations can arise from distinct consequences
of the same image-construction process.

\section{Datasets and Evaluation Benchmarks}
\label{appendix_data}

All benchmarks use data from the OpenProblems in Single-Cell Analysis platform with standardized preprocessing~\citep{luecken2025defining}.
\begin{table}[!htbp]
  \centering 
  \caption{Overview of single-cell datasets. \textit{Imbalance} denotes the ratio between the majority and minority class sizes. The NeurIPS 2021 dataset is used for cross-modality prediction (regression), where the target dimension corresponds to the number of surface proteins.}
  \label{tab:datasets} 
  \resizebox{\columnwidth}{!}{
  \begin{tabular}{lccccc}
  \toprule 
  \textbf{Dataset} & \textbf{Samples} & \textbf{Genes} & \textbf{Classes/Targets} & \textbf{Batches} & \textbf{Imbalance} \\
  \midrule 
  \multicolumn{6}{l}{\textit{Task 1: Cell Type Classification}} \\
  Diabetic Kidney Disease & 39,176 & 27,980 & 13 & 11 & 54.7 \\
  GTEx v9 & 209,126 & 28,094 & 53 & 16 & 956.0 \\
  Immune Cell Atlas & 329,762 & 29,335 & 35 & 12 & 964.4 \\
  HypoMap & 384,925 & 41,642 & 13 & 24 & 5,350.2 \\
  Tabula Sapiens & 483,152 & 55,904 & 160 & 29 & 11,734.7 \\
  \midrule
  \multicolumn{6}{l}{\textit{Task 2: Cross-Modality Prediction (Gene-to-Protein)}} \\
  NeurIPS 2021 CITE-Seq & 8,617 & 36,280 & 13 (Proteins) & 12 & N/A \\
  \bottomrule 
\end{tabular}
  } 
\end{table}
\subsection{Dataset Specifications}
\label{app:datasets}

The Diabetic Kidney Disease (DKD) dataset~\citep{wilson2022multimodal} profiles human kidney cortex from healthy and diabetic donors, including subtypes such as injured proximal tubule cells. GTEx v9~\citep{eraslan2022single} contains 209,126 nuclei from eight tissues and 16 donors, with technical variability and donor-specific batch effects. The Immune Cell Atlas (ICA)~\citep{dominguez2022cross} covers immune cells across 16 tissues, including tissue-resident and circulating populations. HypoMap~\citep{steuernagel2022hypomap} provides a murine hypothalamus reference of approximately 385,000 cells with diverse neuronal subtypes. Tabula Sapiens (TS)~\citep{the2022tabula} contains nearly 500,000 cells from 24 organs. Its class-imbalance ratio exceeds $10,000\times$, making rare-cell performance particularly important. Table~\ref{tab:datasets} reports the ratio of the largest to smallest class for each dataset.

The NeurIPS 2021 CITE-Seq (GEX2ADT) dataset~\citep{stoeckius2017simultaneous} contains bone marrow mononuclear cells from 12 healthy donors. The task is to predict 13 surface-protein abundances from gene expression. Multiple collection sites per donor create a nested batch structure, requiring models to distinguish proteomic signal from technical variation.

Figures~\ref{fig:hvg_cell_vision} and \ref{fig:full_cell_vision} compare DKD images by cell type (rows) and mapping (columns), using the viridis colormap for single-channel images. They show the smooth clusters of DeepInsight and the finer texture of Random Mapping examined in the spectral analysis.

\begin{figure}[p]
  \centering
    \includegraphics[width=0.95\linewidth]{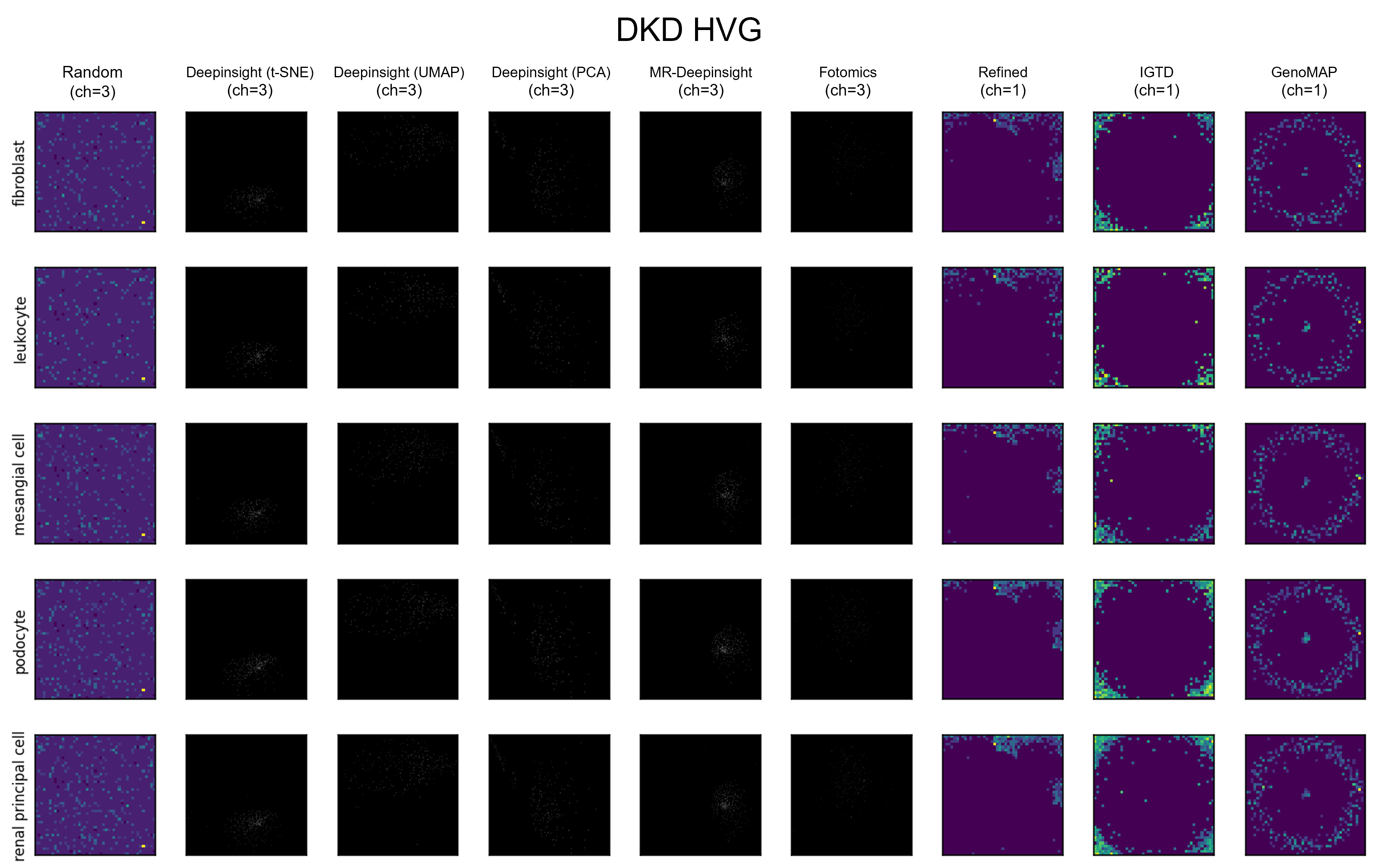}
    \caption{
      Visual comparison of gene expression-to-image mappings on DKD highly variable genes. Each row represents a specific cell type, while each column corresponds to a mapping strategy.
    } 
    \label{fig:hvg_cell_vision}
  \medskip
    \includegraphics[width=0.95\linewidth]{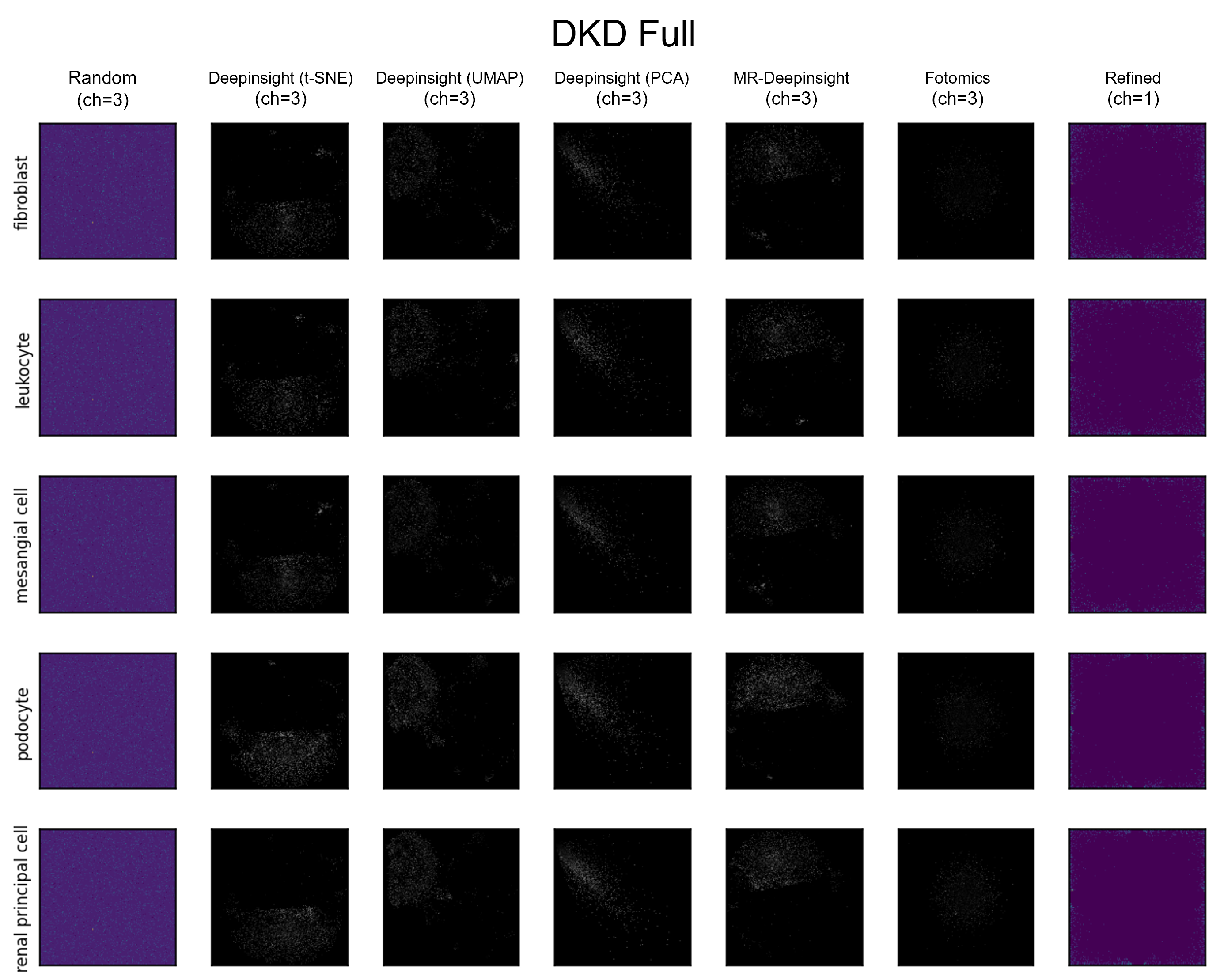}
    \caption{
      Visual comparison of gene expression-to-image mappings on DKD full genomes. Each row represents a specific cell type, while each column corresponds to a mapping strategy.
    } 
    \label{fig:full_cell_vision}
\end{figure}

\subsection{Evaluation Metrics}
\label{app:metrics}

For cell-type annotation, we report accuracy and F1-Macro. We prioritize F1-Macro because it weights each cell type equally, whereas accuracy can be dominated by common classes in these imbalanced atlases.

For gene-to-protein regression, we report root mean squared error (RMSE), global Pearson correlation, and Spearman correlation. RMSE measures prediction error, Pearson correlation captures linear association, and Spearman correlation measures rank association. The rank-based metric complements Pearson correlation when expression scales are nonlinear or affected by extreme values.

\section{Detailed Experimental Settings}

We detail the mapping implementations, baseline configurations, and training protocol below.
\subsection{Gene Expression-to-Image (GE2I) Transformation Protocols}
\label{appendix_GE2I_protocols}

We adapt the official method implementations to a shared pipeline.

\textbf{Preprocessing and Algorithm Hyperparameters.} 

Data preprocessing followed method-specific protocols to maintain consistency with original implementations. In general, gene expression counts were log-normalized ($\log(1+x)$) and globally scaled to the $[0, 1]$ interval. For the \textbf{DeepInsight} family, we followed the original DeepInsight-style preprocessing and used a MinMax scaler, which demonstrated greater stability for manifold projections.

For algorithm configurations, \textbf{DeepInsight} variants used \texttt{perplexity=30} for t-SNE, the default setting for PCA, and \texttt{n\_neighbors=15} with \texttt{min\_dist=0.1} for UMAP, employing cosine metrics and mean-aggregation for pixel collisions. \textbf{MR-DeepInsight} applied identical parameters across parallel views (\texttt{n\_jobs=8}) with majority voting. \textbf{Fotomics} utilized a LogScaler and FFT-based assignment with mean-aggregation.

For adaptive approaches targeting square grids with side length $S = \lceil\sqrt{G}\rceil$, \textbf{Genomap} applied Gromov-Wasserstein Optimal Transport with entropic regularization ($\epsilon=0.1$, \texttt{num\_iter=200}) to generate weighted pixel superpositions. \textbf{IGTD} used the Error Reduction Repositioning (ERR) algorithm with \texttt{max\_step=30,000} and Pearson correlation to enforce a strictly one-to-one mapping. \textbf{REFINED} employed MDS initialization followed by hill-climbing with nearest-neighbor slotting to resolve conflicts. \textbf{Random Mapping} provided a baseline one-to-one assignment on the adaptive grid without semantic priors.

\textbf{Resolution Normalization and Spectral Analysis.} 

We compute SE and LFER on raw generator outputs at their native resolution ($S \times S$ for adaptive grids), before resizing. This avoids suppressing high-frequency structure, particularly in Random Mapping, through interpolation.

For CNN training, we resize maps to $224 \times 224$ using bilinear interpolation with antialiasing, then normalize each channel using training-set means and standard deviations. Spectral descriptors therefore characterize the native maps rather than the resized CNN inputs.

\subsection{Collision-Aware Unique Assignment}
\label{app:unique_assignment}
For mappings that produce continuous 2D coordinates, we replace conventional binning with a greedy assignment to distinct pixels. Let $\mathbf{z}_i\in\mathbb{R}^2$ be the scaled coordinate of gene $i$, and assume that the canvas has sufficient capacity, $|\Omega|=HW\geq G$. Rounding and boundary clipping define its target pixel:
\begin{equation}
\mathbf{t}_i=\Pi_{\Omega}\!\left(\operatorname{round}(\mathbf{z}_i)\right),
\end{equation}
where $\Pi_{\Omega}$ clips each coordinate to the valid grid range. The corresponding global assignment objective is
\begin{equation}
\min_{\mathbf{a}_1,\ldots,\mathbf{a}_G\in\Omega}
\sum_{i=1}^{G}\|\mathbf{a}_i-\mathbf{t}_i\|_2^2
\quad\text{subject to}\quad
\mathbf{a}_i\neq\mathbf{a}_j\quad\text{for all }i\neq j.
\label{eq:unique_assignment_objective}
\end{equation}
Rather than solve this objective jointly, we process genes in a fixed order. Starting with $A_0=\varnothing$, each gene is assigned to the nearest unoccupied pixel:
\begin{equation}
\mathbf{a}_i\in\operatorname*{arg\,min}_{p\in\Omega\setminus A_{i-1}}
\|p-\mathbf{t}_i\|_2^2,
\qquad A_i=A_{i-1}\cup\{\mathbf{a}_i\}.
\end{equation}

We build a \texttt{cKDTree} over the grid coordinates and maintain an occupancy mask. For each target, we query nearest neighbors with geometrically increasing candidate counts, $k=1,2,4,\ldots$, capped at $|\Omega|$. Occupied candidates are discarded; once free candidates are available, we select the closest one. If the adaptive search returns no valid candidate, an exhaustive search over the remaining free pixels provides a fallback. This avoids constructing a dense gene--pixel cost matrix, although searches can become expensive as the grid fills.

Provided $HW\geq G$, the procedure guarantees an injective assignment and hence $CR=0$ and $APPG=1$ on the native single-view grid. It minimizes displacement from the rounded target at each step, conditional on earlier assignments; it does not guarantee a global optimum of Eq.~\ref{eq:unique_assignment_objective} or preservation of the original neighborhoods. The result depends on gene order and tie-breaking. We use this heuristic as a collision-removal intervention distinct from exact LSA, with its cross-dataset results reported in Table~\ref{tab:lsa-vs-bin-resnet18}.

\subsection{Baseline Model Configurations}
\label{appendixa2}
We follow the hyperparameters and input protocols of OpenProblems~\citep{luecken2025defining}. XGBoost~\citep{chen2016xgboost} uses precomputed PCA features, typically 50 components derived from HVGs, rather than raw counts. Seurat uses its default variable-selection, scaling, and classification workflow~\citep{hao2021integrated}.

\textbf{Generative and foundation models.}
For latent modeling, \textbf{scANVI} and its reference-mapping variant \textbf{scANVI+scArches} were configured with two hidden layers ($n_{hidden}=128$), a 30-dimensional latent space, and a dropout rate of 0.2~\citep{lotfollahi2020query}. Training was conducted for approximately 100 to 200 epochs with batch covariates enabled.
We evaluate scGPT~\citep{cui2024scgpt} and CellFM~\citep{zeng2025cellfm} using zero-shot embeddings. The evaluation uses the pretrained \textit{scGPT\_human} and \textit{scGPT\_80M} checkpoints with a weighted KNN classifier ($k=10$) on the embeddings.

\subsection{Standardized Training and Visualization}
\label{appendixa3}
All GE2I representations use a shared ResNet-18 probe initialized with standard weights. We train with AdamW (initial learning rate $1 \times 10^{-3}$, weight decay $0.01$), a 5-epoch linear warmup, and cosine annealing for up to 50 epochs. Early stopping uses a patience of 10 epochs.

\section{Cross-Modality Gene-to-Protein Prediction}
\label{sec:appendix_cross_modality}

We evaluate GE2I on CITE-Seq gene-to-protein prediction~\citep{stoeckius2017simultaneous} using full-transcriptome inputs. Each cell's gene expression profile is represented as an image and used to predict the abundances of 13 surface proteins. Dataset specifications and image-construction procedures are provided in Appendices~\ref{app:datasets} and~\ref{appendix_GE2I_protocols}, respectively. We report root mean squared error (RMSE), overall Pearson correlation (PCC), and overall Spearman correlation to assess prediction error, linear association, and rank agreement with the measured protein abundances.

\begin{table}[!htbp]
\centering
\caption{Gene-to-protein prediction on CITE-Seq using full-transcriptome inputs. Results are reported as mean $\pm$ standard deviation. Lower RMSE and higher correlations indicate better prediction.}
\label{tab:cite_seq_cbmc_cross_modality}
\small
\begin{tabular}{lccc}
\toprule
Method & RMSE $\downarrow$ & Overall PCC $\uparrow$ & Overall Spearman $\uparrow$ \\
\midrule
Random Mapping       & $0.3688\pm0.00$ & $0.7996\pm0.00$ & $0.7443\pm0.01$ \\
DeepInsight (t-SNE)   & $0.3667\pm0.00$ & $0.8009\pm0.00$ & $0.7451\pm0.01$ \\
DeepInsight (UMAP)    & $0.4397\pm0.05$ & $0.7497\pm0.02$ & $0.7266\pm0.02$ \\
DeepInsight (PCA)     & $0.3741\pm0.02$ & $0.7995\pm0.01$ & $0.7647\pm0.01$ \\
MR-DeepInsight       & $0.3479\pm0.00$ & $0.8237\pm0.01$ & $0.7797\pm0.00$ \\
Fotomics             & $0.4125\pm0.01$ & $0.7404\pm0.01$ & $0.7076\pm0.00$ \\
REFINED              & $0.3999\pm0.02$ & $0.7726\pm0.01$ & $0.7214\pm0.01$ \\
\bottomrule
\end{tabular}
\end{table}

Table~\ref{tab:cite_seq_cbmc_cross_modality} reports the results for the evaluated image mappings. Both random and structured layouts produce predictions correlated with measured protein abundances. DeepInsight (t-SNE) and Random Mapping achieve similar overall PCC values of $0.8009$ and $0.7996$, with RMSE values of $0.3667$ and $0.3688$, respectively. This evaluation extends the use of gene expression images from predicting categorical cell identities to estimating continuous protein measurements.

\section{Additional Ablations}

This section collects supplementary result tables for the expression-space
baseline and the visual training choices analyzed in the main text.

\subsection{Full-Genome MLP Ablation}

\begin{table}[!htbp]
\centering
\caption{Tuned full-genome Simple MLP baseline on raw expression vectors. Results use the same splits as the GE2I classification benchmarks.}
\label{tab:simple-mlp-full}
\begin{tabular}{lccccc}
\toprule
Metric & DKD & GTEx v9 & HypoMap & ICA & TS \\
\midrule
Acc & $0.9660_{\pm0.0011}$ & $0.9078_{\pm0.0025}$ & $0.9927_{\pm0.0027}$ & $0.9002_{\pm0.0017}$ & $0.4941_{\pm0.1406}$ \\
F1-M & $0.9312_{\pm0.0161}$ & $0.3624_{\pm0.0239}$ & $0.6903_{\pm0.2036}$ & $0.8285_{\pm0.0165}$ & $0.3044_{\pm0.0524}$ \\
\bottomrule
\end{tabular}
\end{table}

The tuned Simple MLP uses raw expression vectors and the same splits as GE2I. This comparison tests whether an image representation is necessary to benefit from the larger gene set. The MLP is competitive on DKD, ICA, and TS, but its F1-Macro is lower than that of the best structured GE2I methods on GTEx v9 and HypoMap. Raw information can therefore explain some full-genome gains, while projection provides useful inductive bias in other settings. The distinction between metrics is important: on GTEx v9, the MLP attains high accuracy despite weaker F1-Macro, so its aggregate correctness does not imply equally strong performance across cell types. This ablation supports a dataset-dependent role for projection rather than a general superiority of either representation family.

\subsection{Supplementary Visual-Training Results}

The following tables provide the complete backbone and augmentation results
discussed in Section~\ref{sec:cv_priors}.

\begin{table}[!htbp]
\begin{minipage}[t]{0.49\linewidth}
\vspace{0pt}
\centering
\caption{Collision-removal results for DeepInsight (t-SNE) across DKD backbones. Superscripts report relative performance changes.}
\label{tab:dkd-backbone-lsa-shared}
\resizebox{\linewidth}{!}{
\begin{tabular}{lcc}
  \toprule
  Backbone & Acc & F1-M \\
  \midrule
  ResNet-18 &
  $0.9640_{\pm 0.0055}^{\textcolor{imporvement}{+1.2\%}}$ &
  $0.9094_{\pm 0.0410}^{\textcolor{imporvement}{+3.4\%}}$ \\

  ResNet-50 &
  $0.9623_{\pm 0.0065}^{\textcolor{imporvement}{+1.8\%}}$ &
  $0.9071_{\pm 0.0358}^{\textcolor{imporvement}{+4.6\%}}$ \\

  ConvNeXt-Tiny &
  $0.9573_{\pm 0.0044}^{\textcolor{imporvement}{+0.3\%}}$ &
  $0.9022_{\pm 0.0371}^{\textcolor{imporvement}{+4.0\%}}$ \\

  EfficientNet-B0 &
  $0.9654_{\pm 0.0016}^{\textcolor{imporvement}{+1.1\%}}$ &
  $0.9202_{\pm 0.0369}^{\textcolor{imporvement}{+7.0\%}}$ \\

  ViT-B/16 &
  $0.9512_{\pm 0.0084}^{\textcolor{imporvement}{+1.2\%}}$ &
  $0.8618_{\pm 0.0074}^{\textcolor{imporvement}{+4.4\%}}$ \\
  \bottomrule
  \end{tabular}}
\end{minipage}\hfill
\begin{minipage}[t]{0.47\linewidth}
\vspace{0pt}
\centering
\caption{Random mapping performance across backbones on DKD.}
\label{tab:random-backbones-dkd}
\resizebox{\linewidth}{!}{
\begin{tabular}{lcc}
\toprule
Backbone & Acc & F1-M \\
\midrule
ResNet-18 & $0.9609_{\pm 0.00}$ & $0.8635_{\pm 0.02}$ \\
ResNet-50 & $0.9605_{\pm 0.00}$ & $0.8984_{\pm 0.01}$ \\
ConvNeXt-Tiny & $0.9512_{\pm 0.01}$ & $0.8576_{\pm 0.02}$ \\
EfficientNet-B0 & $0.9603_{\pm 0.00}$ & $0.8587_{\pm 0.01}$ \\
ViT-B/16 & $0.9271_{\pm 0.00}$ & $0.8080_{\pm 0.01}$ \\
\bottomrule
\end{tabular}
}
\end{minipage}
\end{table}

\begin{table}[!htbp]
\centering
\caption{\textbf{Image augmentation on DKD.} Each augmentation is evaluated relative to the same mapping without augmentation. All tested configurations, including their combination, are reported.}
\label{tab:augmentation_ablation_dkd}
\resizebox{0.64\linewidth}{!}{
\begin{tabular}{ll c c}
\toprule
Method & Augmentation & ACC & F1-Macro\\
\midrule
\multirow{6}{*}{DeepInsight (t-SNE)} & None & $0.9537_{\pm 0.00}$ & $0.8868_{\pm 0.02}$\\
 & Geometric & $0.9532_{\pm 0.01}$ & $0.8819_{\pm 0.02}$\\
 & Bio. Cutout & $0.9513_{\pm 0.01}$ & $0.8540_{\pm 0.04}$\\
 & Bio. Noise & $0.9482_{\pm 0.01}$ & $0.8771_{\pm 0.01}$\\
 & Bio. Intensity & $0.9566_{\pm 0.00}$ & $0.8938_{\pm 0.01}$\\
 & All & $0.4689_{\pm 0.31}$ & $0.3227_{\pm 0.22}$\\
\midrule
\multirow{6}{*}{Random} & None & $0.9421_{\pm 0.00}$ & $0.8053_{\pm 0.02}$\\
 & Geometric & $0.9550_{\pm 0.01}$ & $0.8416_{\pm 0.04}$\\
 & Bio. Cutout & $0.9532_{\pm 0.01}$ & $0.8319_{\pm 0.03}$\\
 & Bio. Noise & $0.9578_{\pm 0.01}$ & $0.8668_{\pm 0.01}$\\
 & Bio. Intensity & $0.9563_{\pm 0.01}$ & $0.8652_{\pm 0.03}$\\
 & All & $0.9516_{\pm 0.01}$ & $0.8534_{\pm 0.01}$\\
 \midrule
\multirow{6}{*}{Refined} & None & $0.9566_{\pm 0.00}$ & $0.8704_{\pm 0.03}$\\
 & Geometric & $0.9497_{\pm 0.01}$ & $0.8263_{\pm 0.01}$\\
 & Bio. Cutout & $0.9552_{\pm 0.00}$ & $0.8698_{\pm 0.01}$\\
 & Bio. Noise & $0.2926_{\pm 0.22}$ & $0.0932_{\pm 0.12}$\\
 & Bio. Intensity & $0.9584_{\pm 0.00}$ & $0.8837_{\pm 0.02}$\\
 & All & $0.5444_{\pm 0.25}$ & $0.1941_{\pm 0.15}$\\
\bottomrule
\end{tabular}}
\end{table}

\subsection{Interpretability Trade-off}

\begin{table}[!htbp]
\centering
\small
\caption{Grad-CAM interpretability comparison on DKD. Canonical-marker coverage is measured as the number of established markers falling inside the salient region. Higher center mass and marker hit ratio are better; lower entropy and radius are better.}
\label{tab:dkd_gradcam_interpretability}
\resizebox{\columnwidth}{!}{
\begin{tabular}{lccc}
\toprule
Setting & Random & DeepInsight (t-SNE) & Key evidence \\
\midrule
CAM entropy $\downarrow$ & 0.950 & 0.907 & more compact saliency \\
Center mass $\uparrow$ & 0.427 & 0.521 & stronger spatial focus \\
Marker hit ratio $\uparrow$ & 0.25 & 0.50 & better biological alignment \\
\midrule
DCT marker coverage $\uparrow$ & 2/5 & 5/5 & SLC12A3, PVALB, TRPM6, CALB1, FXYD2 \\
TAL marker coverage $\uparrow$ & 2/5 & 4/5 & UMOD, SLC12A1, KCNJ1, CLDN16 \\
Podocyte marker coverage $\uparrow$ & 2/5 & 5/5 & NPHS1, NPHS2, PODXL, WT1, SYNPO \\
\bottomrule
\end{tabular}
}
\end{table}

DeepInsight (t-SNE) yields more concentrated Grad-CAM responses than Random Mapping: CAM entropy is lower ($0.907$ vs.\ $0.950$), center mass higher ($0.521$ vs.\ $0.427$), and mean radius smaller ($0.201$ vs.\ $0.227$). Canonical-marker hit ratio increases from $0.25$ to $0.50$, and mean marker CAM percentile from $0.559$ to $0.742$. The structured layout recovers more canonical markers within hotspots for distal convoluted tubule cells, thick ascending limb cells, and podocytes.

Spatial concentration alone would not establish biological relevance, since the mapping itself affects the geometry of saliency regions. Marker recovery provides complementary evidence that the structured hotspots align with known cell-type features. These results concern the interpretability of the evaluated Grad-CAM maps, not causal validation of the identified genes, and should be read separately from the classification comparison.

\begin{figure}[!htbp]
\centering
\includegraphics[width=\linewidth]{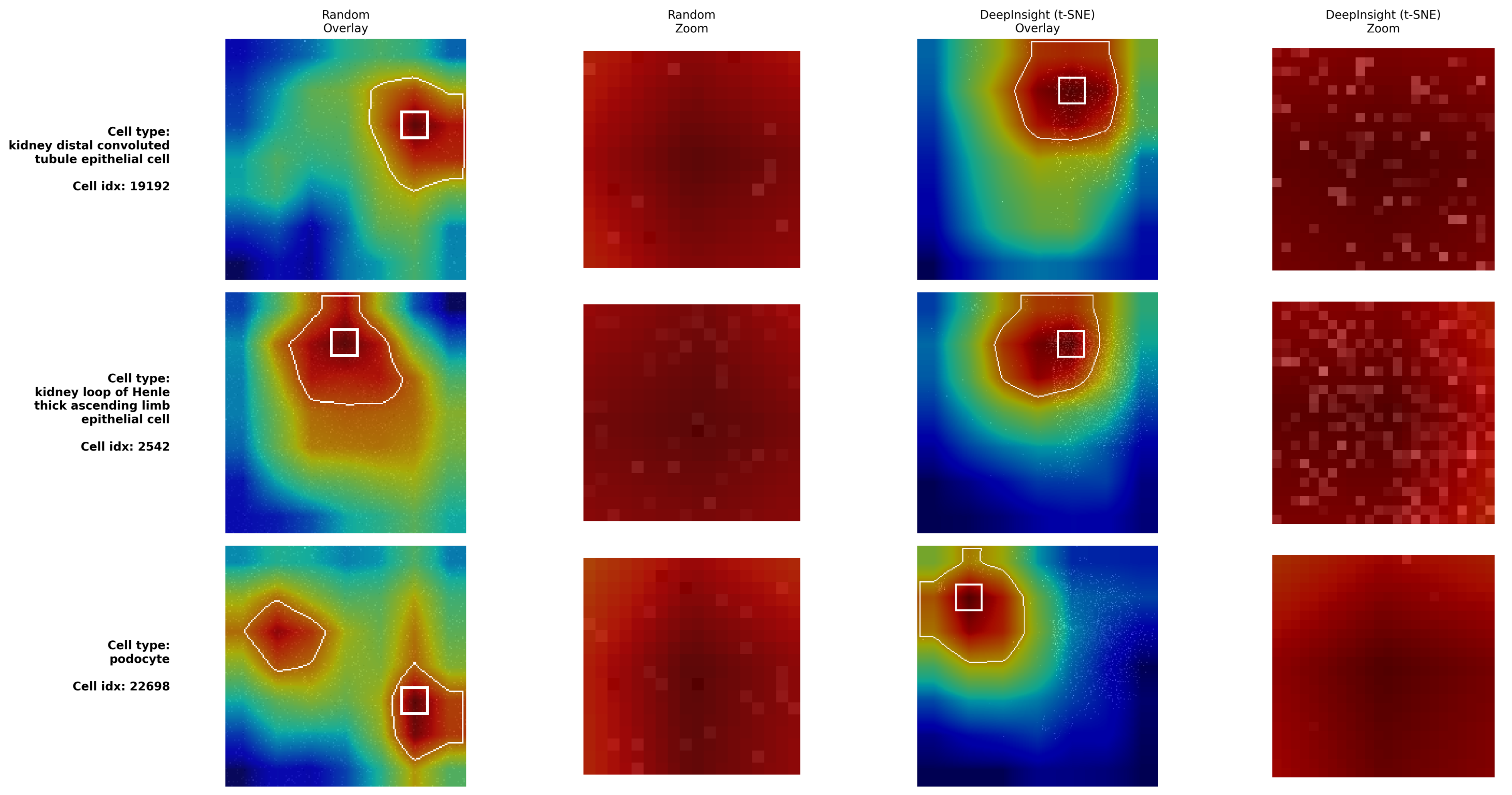}
\caption{Representative Grad-CAM examples for three cell types where \texttt{DeepInsight-tSNE} shows stronger biological alignment than the random mapping. In distal convoluted tubule, thick ascending limb, and podocytes, the structured mapping yields tighter saliency regions and substantially higher canonical-marker recovery within the hotspot.}
\label{fig:dkd_gradcam_examples}
\end{figure}

\FloatBarrier
\section{Comprehensive Metric Breakdown}
Table~\ref{tab:full_metrics} reports spectral descriptors and performance for each method--dataset pair used in Section~\ref{sec:mechanism}.

\textbf{Full-transcriptome manifold layouts create more gene collisions than their HVG counterparts.} CR increases with full-genome inputs; for DeepInsight (t-SNE) on HypoMap, it rises from $0.198$ to $0.641$. Compact gene clusters can thus reduce the number of expression values represented separately.

\textbf{IGTD and REFINED achieve near-maximal smoothness without gene collisions.} Both methods have $CR=0$ and LFER near $1.0$. Their layouts preserve separate gene assignments while concentrating spectral power at low frequencies. This distinguishes collision avoidance from texture preservation and motivates treating both as potential contributors to performance.

\begingroup
\small
\setlength{\tabcolsep}{3pt}
\begin{longtable}{llllllll}
\caption{Detailed metrics for spectral properties, information density, and performance across all datasets.} \label{tab:full_metrics} \\
\toprule
Method & Dataset & APPG & CR & LFER & SE & Acc & F1-M \\
\midrule
\endfirsthead
\multicolumn{8}{c}{\tablename\ \thetable\ -- \textit{Continued from previous page}} \\
\toprule
Method & Dataset & APPG & CR & LFER & SE & Acc & F1-M \\
\midrule
\endhead
\midrule
\multicolumn{8}{r}{\textit{Continued on next page}} \\
\endfoot
\bottomrule
\endlastfoot
\multicolumn{8}{l}{\textbf{Highly Variable Genes (HVG)}} \\
\midrule
Random & DKD & 1.000 & 0.000 & 0.177 & 3.029 & 0.900 & 0.773 \\
Random & GTEx & 1.000 & 0.000 & 0.198 & 3.033 & 0.653 & 0.211 \\
Random & HypoMAP & 1.000 & 0.000 & 0.191 & 3.030 & 0.980 & 0.557 \\
Random & ICA & 1.000 & 0.000 & 0.207 & 3.034 & 0.683 & 0.484 \\
Random & Tabula Sapiens & 1.000 & 0.000 & 0.221 & 3.026 & 0.384 & 0.215 \\
DeepInsight (t-SNE) & DKD & 0.869 & 0.127 & 0.900 & 2.846 & 0.892 & 0.742 \\
DeepInsight (t-SNE) & GTEx & 0.831 & 0.170 & 0.823 & 3.520 & 0.626 & 0.194 \\
DeepInsight (t-SNE) & HypoMAP & 0.832 & 0.198 & 0.933 & 3.088 & 0.983 & 0.572 \\
DeepInsight (t-SNE) & ICA & 0.832 & 0.179 & 0.954 & 2.887 & 0.693 & 0.504 \\
DeepInsight (t-SNE) & Tabula Sapiens & 0.797 & 0.196 & 0.945 & 2.769 & 0.371 & 0.184 \\
DeepInsight (UMAP) & DKD & 0.976 & 0.025 & 0.988 & 0.582 & 0.900 & 0.759 \\
DeepInsight (UMAP) & GTEx & 0.757 & 0.235 & 0.871 & 2.043 & 0.677 & 0.221 \\
DeepInsight (UMAP) & HypoMAP & 0.807 & 0.183 & 0.996 & 0.455 & 0.981 & 0.578 \\
DeepInsight (UMAP) & ICA & 0.778 & 0.205 & 0.953 & 1.254 & 0.705 & 0.510 \\
DeepInsight (UMAP) & Tabula Sapiens & 0.610 & 0.345 & 0.873 & 2.587 & 0.392 & 0.197 \\
DeepInsight (PCA) & DKD & 0.230 & 0.246 & 0.511 & 3.997 & 0.903 & 0.781 \\
DeepInsight (PCA) & GTEx & 0.159 & 0.293 & 0.401 & 4.408 & 0.623 & 0.176 \\
DeepInsight (PCA) & HypoMAP & 0.171 & 0.109 & 0.442 & 4.253 & 0.983 & 0.572 \\
DeepInsight (PCA) & ICA & 0.157 & 0.185 & 0.314 & 4.614 & 0.686 & 0.481 \\
DeepInsight (PCA) & Tabula Sapiens & 0.173 & 0.127 & 0.387 & 4.460 & 0.374 & 0.202 \\
MR-DeepInsight & DKD & 1.784 & 0.103 & 0.895 & 1.768 & 0.908 & 0.797 \\
MR-DeepInsight & GTEx & 1.656 & 0.165 & 0.832 & 2.558 & 0.684 & 0.199 \\
MR-DeepInsight & HypoMAP & 1.641 & 0.165 & 0.919 & 1.686 & 0.992 & 0.593 \\
MR-DeepInsight & ICA & 1.614 & 0.178 & 0.909 & 2.104 & 0.748 & 0.639 \\
MR-DeepInsight & Tabula Sapiens & 1.565 & 0.186 & 0.917 & 1.894 & 0.387 & 0.210 \\
FOTOMICS & DKD & 0.863 & 0.140 & 0.620 & 3.510 & 0.902 & 0.742 \\
FOTOMICS & GTEx & 0.919 & 0.086 & 0.450 & 4.189 & 0.658 & 0.201 \\
FOTOMICS & HypoMAP & 0.885 & 0.116 & 0.569 & 3.738 & 0.981 & 0.520 \\
FOTOMICS & ICA & 0.930 & 0.074 & 0.611 & 3.444 & 0.707 & 0.510 \\
FOTOMICS & Tabula Sapiens & 0.927 & 0.077 & 0.636 & 3.375 & 0.383 & 0.209 \\
Refined & DKD & 1.000 & 0.000 & 0.995 & 0.555 & 0.902 & 0.753 \\
Refined & GTEx & 1.000 & 0.000 & 0.991 & 0.545 & 0.637 & 0.194 \\
Refined & HypoMAP & 1.000 & 0.000 & 0.997 & 0.522 & 0.981 & 0.560 \\
Refined & ICA & 1.000 & 0.000 & 0.995 & 0.608 & 0.663 & 0.484 \\
Refined & Tabula Sapiens & 1.000 & 0.000 & 0.993 & 0.646 & 0.373 & 0.215 \\
IGTD & DKD & 1.000 & 0.000 & 0.996 & 0.467 & 0.902 & 0.733 \\
IGTD & GTEx & 1.000 & 0.000 & 0.994 & 0.481 & 0.666 & 0.211 \\
IGTD & HypoMAP & 1.000 & 0.000 & 1.000 & 0.422 & 0.984 & 0.539 \\
IGTD & ICA & 1.000 & 0.000 & 0.999 & 0.429 & 0.690 & 0.500 \\
IGTD & Tabula Sapiens & 1.000 & 0.000 & 0.996 & 0.463 & 0.393 & 0.207 \\
Genomap & DKD & 1.000 & 1.000 & 0.776 & 1.877 & 0.881 & 0.772 \\
Genomap & GTEx & 1.000 & 1.000 & 0.411 & 2.815 & 0.644 & 0.207 \\
Genomap & HypoMAP & 1.000 & 1.000 & 0.610 & 2.400 & 0.977 & 0.542 \\
Genomap & ICA & 1.000 & 1.000 & 0.712 & 2.221 & 0.668 & 0.474 \\
Genomap & Tabula Sapiens & 1.000 & 1.000 & 0.833 & 1.608 & 0.361 & 0.210 \\
\midrule
\multicolumn{8}{l}{\textbf{Full-Transcriptome}} \\
Random & DKD & 1.000 & 0.000 & 0.252 & 4.409 & 0.942 & 0.805 \\
Random & GTEx & 1.000 & 0.000 & 0.244 & 4.409 & 0.849 & 0.460 \\
Random & HypoMAP & 1.000 & 0.000 & 0.241 & 4.607 & 0.990 & 0.655 \\
Random & ICA & 1.000 & 0.000 & 0.233 & 4.433 & 0.879 & 0.731 \\
Random & Tabula Sapiens & 1.000 & 0.000 & 0.241 & 4.755 & 0.457 & 0.345 \\
DeepInsight (t-SNE) & DKD & 0.489 & 0.482 & 0.943 & 1.958 & 0.954 & 0.887 \\
DeepInsight (t-SNE) & GTEx & 0.446 & 0.504 & 0.867 & 2.643 & 0.855 & 0.506 \\
DeepInsight (t-SNE) & HypoMAP & 0.293 & 0.641 & 0.950 & 2.205 & 0.992 & 0.749 \\
DeepInsight (t-SNE) & ICA & 0.402 & 0.581 & 0.954 & 2.036 & 0.877 & 0.748 \\
DeepInsight (t-SNE) & Tabula Sapiens & 0.255 & 0.644 & 0.946 & 2.336 & 0.427 & 0.247 \\
DeepInsight (UMAP) & DKD & 0.549 & 0.440 & 0.904 & 2.743 & 0.957 & 0.892 \\
DeepInsight (UMAP) & GTEx & 0.262 & 0.655 & 0.802 & 3.204 & 0.846 & 0.489 \\
DeepInsight (UMAP) & HypoMAP & 0.190 & 0.766 & 0.916 & 2.878 & 0.994 & 0.771 \\
DeepInsight (UMAP) & ICA & 0.358 & 0.635 & 0.818 & 3.172 & 0.882 & 0.763 \\
DeepInsight (UMAP) & Tabula Sapiens & 0.070 & 0.795 & 0.917 & 3.008 & 0.428 & 0.267 \\
DeepInsight (PCA) & DKD & 0.116 & 0.422 & 0.765 & 2.791 & 0.958 & 0.846 \\
DeepInsight (PCA) & GTEx & 0.080 & 0.480 & 0.650 & 3.228 & 0.842 & 0.442 \\
DeepInsight (PCA) & HypoMAP & 0.084 & 0.460 & 0.734 & 2.677 & 0.993 & 0.728 \\
DeepInsight (PCA) & ICA & 0.100 & 0.405 & 0.644 & 3.342 & 0.872 & 0.748 \\
DeepInsight (PCA) & Tabula Sapiens & 0.048 & 0.447 & 0.753 & 2.715 & 0.442 & 0.305 \\
MR-DeepInsight & DKD & 0.876 & 0.572 & 0.891 & 2.436 & 0.970 & 0.927 \\
MR-DeepInsight & GTEx & 0.714 & 0.548 & 0.785 & 3.186 & 0.871 & 0.366 \\
MR-DeepInsight & HypoMAP & 0.557 & 0.648 & 0.896 & 2.574 & 0.999 & 0.802 \\
MR-DeepInsight & ICA & 0.625 & 0.679 & 0.901 & 2.481 & 0.899 & 0.822 \\
MR-DeepInsight & Tabula Sapiens & 0.284 & 0.726 & 0.904 & 2.693 & 0.410 & 0.253 \\
FOTOMICS & DKD & 0.288 & 0.636 & 0.900 & 1.767 & 0.954 & 0.852 \\
FOTOMICS & GTEx & 0.195 & 0.696 & 0.798 & 2.547 & 0.833 & 0.418 \\
FOTOMICS & HypoMAP & 0.249 & 0.672 & 0.871 & 1.882 & 0.991 & 0.719 \\
FOTOMICS & ICA & 0.287 & 0.577 & 0.866 & 2.009 & 0.861 & 0.720 \\
FOTOMICS & Tabula Sapiens & 0.196 & 0.610 & 0.810 & 2.296 & 0.460 & 0.295 \\
Refined & DKD & 1.000 & 0.000 & 0.999 & 0.456 & 0.957 & 0.870 \\
Refined & GTEx & 1.000 & 0.000 & 0.995 & 0.523 & 0.847 & 0.488 \\
Refined & HypoMAP & 1.000 & 0.000 & 1.000 & 0.427 & 0.992 & 0.727 \\
Refined & ICA & 1.000 & 0.000 & 1.000 & 0.432 & 0.883 & 0.730 \\
Genomap & DKD & 1.000 & 1.000 & 1.000 & 1.004 & 0.408 & 0.092 \\
\end{longtable}
\endgroup

\section{Statistical Analysis Diagnostics}
\label{sec:appendix_stats}
\textbf{Clustered inference accounts for repeated seeds, but full-transcriptome coverage is uneven.} A cluster is one dataset--method pair, and $G$ denotes the number of such clusters. We compute CR1 cluster-robust standard errors across repeated seeds. The reported $z^*$ is the coefficient divided by its clustered standard error. Two-sided $p$-values use a $t$ distribution with $G-1$ degrees of freedom, written $t_{G-1}$: 34 degrees for the 35 Full clusters and 44 degrees for the 45 HVG clusters. The $q$-values are Benjamini--Hochberg-adjusted $p$-values over the 16 descriptor tests across both outcomes and input scopes. This adjustment accounts for testing many coefficients simultaneously; $q<0.05$ marks results that remain significant at a 5\% false-discovery-rate threshold.

The Full and HVG models use 105 and 135 runs, respectively. Of 270 potential runs, 243 contain both outcomes and 240 also contain all four descriptors. Thirty Full and 40 HVG runs use APPG/CR values imputed by the observed-seed mean within the same dataset--method--scope group. Full coverage is uneven: Genomap contributes three DKD runs, REFINED contributes 12 runs, and IGTD has no complete records. The reported standard errors treat the supplied imputed values as fixed.

\textbf{HVG conclusion: no descriptor reaches $q<0.05$ for F1-Macro; for accuracy, APPG is positive and CR is negative.} LFER and SE do not reach the adjusted threshold (Table~\ref{tab:hvg_stats}). Compared with the Full results, the HVG pattern indicates that the observed descriptor--performance relationships depend on how many genes are represented.

\begin{center}
\begin{minipage}{\linewidth}
\centering
\captionof{table}{\textbf{Beta regression in the HVG setting.} Descriptors enter jointly with dataset-by-seed fixed effects, and SE is clustered by dataset--method. The $p$-values use a $t$ distribution with 44 degrees of freedom; $q$ is the Benjamini--Hochberg-adjusted $p$-value.}
\label{tab:hvg_stats}
\small
\begin{tabular}{llrrrrr}
\toprule
Target & Descriptor & Coef. & Cluster SE & $z^*$ & $p$ & $q$ \\
\midrule
Accuracy & APPG &  0.146 & 0.044 &  3.34 & 0.002 & 0.004 \\
         & CR   & -0.097 & 0.030 & -3.26 & 0.002 & 0.004 \\
         & LFER &  0.061 & 0.061 &  1.00 & 0.320 & 0.394 \\
         & SE   &  0.022 & 0.016 &  1.40 & 0.167 & 0.223 \\
\midrule
F1-Macro & APPG &  0.172 & 0.091 &  1.88 & 0.066 & 0.096 \\
         & CR   & -0.020 & 0.053 & -0.39 & 0.701 & 0.747 \\
         & LFER &  0.014 & 0.076 &  0.19 & 0.852 & 0.852 \\
         & SE   &  0.016 & 0.022 &  0.72 & 0.473 & 0.541 \\
\bottomrule
\end{tabular}
\end{minipage}
\end{center}

\textbf{Full-transcriptome spectral descriptors are collinear, but their positive associations survive key exclusions.} After accounting for fixed effects, the Full VIFs for APPG, CR, LFER, and SE are $1.96$, $3.41$, $11.57$, and $8.43$; the HVG values are $1.27$, $1.01$, $2.22$, and $2.47$. The higher Full LFER and SE values show that the two spectral coefficients are best interpreted together. Their positive Full associations remain when using only directly measured descriptors or excluding Genomap. Allowing dataset-specific precision reduces the APPG and CR associations, while method fixed effects attenuate the spectral associations by focusing the comparison on variation within each method.

\end{document}